\documentclass{article} % For LaTeX2e

\usepackage{hyperref}
\usepackage{xurl}

\usepackage{amsmath,amssymb,mathtools,bm}
 \usepackage{natbib} 
\usepackage{booktabs,tabularx,array,longtable}
\usepackage{microtype}
\usepackage{xcolor}
\usepackage{algorithm}
\usepackage{algpseudocode}
\usepackage{amsthm}
\usepackage{graphicx}
\usepackage{tikz}
\usetikzlibrary{arrows.meta}
\usepackage{float}
\usepackage{comment}

\theoremstyle{plain}
\newtheorem{theorem}{Theorem}[section]
\newtheorem{lemma}[theorem]{Lemma}

\newtheorem{definition}[theorem]{Definition}
\newtheorem{assumption}[theorem]{Assumption}
\newtheorem{example}[theorem]{Example}

\newtheoremstyle{italicremark}
  {\topsep}{\topsep}{\itshape}{}{\itshape}{.}{.5em}{}
\theoremstyle{italicremark}
\newtheorem{remark}[theorem]{Remark}
\newtheorem*{remark*}{Remark}

\newcommand{\E}{\mathbb{E}}
\title{ When Should a Human Take Back Control?\\
Optimal Delegation under Turbulent AI Risk }

\author{Haoze~\textsc{Yan}\thanks{Department of Industrial Engineering and Operations Research, UC Berkeley, United States. \texttt{haoze.yan@berkeley.edu}},\;
Julien~\textsc{Roze}\thanks{Stoa, United States. \texttt{jroze@stoa.insure}},\;
Ved~\textsc{Upadhyay}\thanks{Stoa, United States. \texttt{ved@stoa.insure}},\;
Unal~\textsc{Tatar}\thanks{Department of Cybersecurity, University at Albany, SUNY, United States. \texttt{utatar@albany.edu}},\;
Thibaut~\textsc{Mastrolia}\thanks{Department of Industrial Engineering and Operations Research, UC Berkeley, United States. \texttt{mastrolia@berkeley.edu}}
}

\begin{document}

\maketitle

% ---------------------------------------------------------------
% Abstract
% ---------------------------------------------------------------
\begin{abstract}
When should a human take control from an autonomous AI agent, how much checking is worthwhile, and when should autonomy resume? These decisions become harder when incidents cluster and recent errors signal elevated risk. Exploratory analyses of AgentLogs coding-agent traces and ODCV-Bench task runs show higher held-out log likelihood for Hawkes incident-history models than matched no-history baselines, with positive gains whose conditional $90\%$ intervals exclude zero. AgentLogs evaluates continuous-time incident timing and type; ODCV evaluates failure occurrence at discrete decisions. Motivated by this predictive evidence, we formulate a continuous-time delegation problem that jointly chooses productive action, verification effort, and switching rates between human-directed, AI-assisted operation and autonomous AI operation. The model assumes that verification weakens incident propagation and balances project benefits against operating, supervision, risk, and handoff costs. Our Hawkes-PPO method summarizes observed incidents with exponential memory filters and learns these decisions from rewards. In synthetic experiments, it achieves a higher evaluated objective than either fixed regime and PPO without history filters, and approaches an approximate full-information benchmark. Further experiments show how incident excitation and supervision costs shape delegation.
\end{abstract}

% Keyword: AI Safety, AI Hallucination, AI risk modeling, Reinforcement Learning, Hawkes-PPO, Stochastic Control

% ---------------------------------------------------------------
% Section 1
% ---------------------------------------------------------------

\section{Introduction}
\label{sec:intro}

Language models are increasingly deployed as agents that retrieve information, plan, invoke tools, and act on the results over extended horizons. This shift makes the reliability of extended agent execution an increasingly important safety concern and revives the problem of scalable oversight \citep{engels2025scaling}: how a limited and costly human supervisor should allocate attention to a system that is cheaper, faster, and possibly more capable, but imperfectly reliable. The difficulty is that agent failures can be dependent across successive steps. A coding agent that misreads a failing test edits the wrong file, and the misleading output of that edit degrades its next diagnosis. For a single evolving task, the operational question is:
\begin{quote}\itshape
When should a costly but reliable human controller delegate a continuously evolving project to a cheaper black-box AI controller, and when should the human take control back because accumulated incident risk has become too large?
\end{quote}

\paragraph{\textbf {Turbulent AI risk.}}

Recent controlled experiments provide evidence that earlier errors can increase the likelihood of subsequent errors: an early incorrect claim can be reused as a premise and induce further incorrect claims \citep{zhang2024snowball}, and injecting erroneous prior steps into a model's context lowers its accuracy on later steps, an effect that persisted across model sizes and was absent in the tested thinking-enabled variants \citep{sinha2026illusion}. Execution traces suggest cross-excitation mechanisms, where a hallucinated state assessment induces an incorrect action that alters the task history \citep{yang2024sweagent,cemri2025mast,liu2026agenthallubenchmarkingautomatedhallucination}.

An erroneous output thus contaminates the information the agent uses next, so past incidents raise the likelihood of further incidents, of the same type (self-excitation) or of another type (cross-excitation); effective supervision can weaken this dependence; and supervision is costly. We use the term turbulent AI risk for incident dynamics with this incident-driven dependence across time: cascade of self-excited incidents, propagating into a cyber-system while raising uncertainty of the outcome. 

\paragraph{Memory-dependent and cascade modeling.}
These observations motivate comparing incident-history models with matched history-independent baselines. Section~\ref{sec:empirical-motivation} provides exploratory predictive comparisons on two archived datasets. We therefore represent incidents by a self- and mutually exciting Hawkes intensity \citep{hawkes1971} which encompasses the memory-dependent impact and cascade AI risk and we treat this specification as a falsifiable reduced-form hypothesis. That the influence of a past incident decays, and that supervision weakens propagation and cascade effect, are further assumptions. The closest support for this assumption is that external feedback, including human-written feedback, improved code repair in the settings evaluated by \cite{olausson2024selfrepair}, with gains depending on feedback quality, whereas intrinsic self-correction did not improve reasoning accuracy in \cite{huang2024large}.

The model adds two simplifying assumptions. First, it holds the severity distribution of each incident fixed, so excitation changes how often incidents occur rather than how large they are, although expected losses still rise with intensity. Then, it assumes an additional disturbance to the agent's output by increasing uncertainty and volatility with incident intensity, so that active periods make outcomes less predictable beyond their direct losses leading to the definition of turbulent incidents.  For example, \cite{hao2024quantifying} shows that hallucinated content can propagate through a system and its spread is itself uncertain. We study observable, history-dependent incidents whose propagation can be reduced through verification and corrective supervision, with state and action turbulence (e.g. hallucination) in the broad agentic sense of Section~\ref{sec:turbulenceHawkes} as the worked formulation; the framework optimizes the response to incident signals observed by the supervisor.

Deployed systems show the decision a supervisor faces. In July 2025, a Replit coding agent deleted a production database despite repeated instructions restricting changes, then incorrectly reported that recovery was impossible, after earlier fabricating data and test results.\footnote{See \url{https://www.saastr.com/replits-new-release-address-most-of-the-challenges-we-hit-vibe-coding-but-is-prosumer-vibe-coding-really-ready-for-commercial-apps-yet/}} In Anthropic's month-long trial of a model running a small shop, a fabricated interaction was followed by related confabulations, human corrections, and eventual recovery, with neither trigger nor recovery explained \citep{anthropic2025vend}. These accounts motivate treating earlier incidents as warning signals that can prompt a reassessment of delegation, the decision our framework formalizes. Permissions and environment separation remain the first line of defense; adaptive supervision addresses the residual decisions they leave open.

Existing oversight protocols can condition on history but do not model incident clustering when deciding how much to delegate \citep{greenblatt2024control,bhatt2025ctrlz}, and the sequential delegation formulations reviewed in Section~\ref{sec:related} below do not represent explicit incident-history excitation, or its effort-dependent suppression, jointly with the delegation decision. Modeling both matters because supervision changes future risk by preventing incidents and their downstream effects, so incident history bears on how long to supervise and when to delegate again; this is the mechanism the framework studies. We formulate the problem as continuous-time stochastic control with two regimes: a more reliable but costly human-assisted regime, in which the human acts on the project and chooses a verification effort that weakens propagation, and autonomous execution by a frozen black-box AI policy that is not optimized by the manager and ignores the turbulence it generates. Delegation and takeover are relaxed switching decisions in which the human controls transition rates rather than transition times, with switching costs that discourage repeated handoffs.

\paragraph{Contributions.}
(i) Empirical evidence for incident-history prediction: exploratory held-out comparisons on AgentLogs and ODCV-Bench show gains over matched no-history baselines using continuous-time intensity and decision-time Bernoulli likelihoods, respectively (Section~\ref{sec:empirical-motivation}). (ii) A supervisory control formulation that jointly chooses delegation, takeover, productive action, and verification effort under assumed effort-dependent incident dynamics, balancing project benefits against operating costs, switching costs, and quadratic-variation risk. (iii) Hawkes-PPO, which combines exponential incident-memory filters with reward-based policy learning, supported by approximation results for the underlying control problem (Appendix~\ref{app:markovian-approximation}). (iv) Synthetic comparisons against fixed regimes, PPO without filters, and an approximate full-information benchmark, together with experiments showing how risk and costs shape supervision.

From matching starts, our Hawkes-PPO learned policy improves on both fixed regimes by point estimates, reducing the evaluated gaps to the approximate full-information benchmark from each matching start (Appendix \ref{app:other-rl}, Table \ref{app:all-values}). The pure-human regime attains the lowest volatility and the highest terminal quality at the highest operating cost, the pure-AI regime the reverse, and the learned policy lies between them. We provide a comparative-statics study shows how the learned switching policy, and the share of the horizon spent under human control, respond to excitation strength, human operating cost and the switching fee.

Section~\ref{sec:related} positions the work, Section~\ref{sec:model} introduces the model, Section~\ref{sec:numerical-ai} the numerical study, and Section~\ref{sec:conclusion} concludes.

% =====================================================================
\section{Related Work}
\label{sec:related}

\paragraph{Failure propagation in language-model agents.}
Beyond the snowballing and self-conditioning experiments cited above \citep{zhang2024snowball,sinha2026illusion}, studies of self-repair report that a model's own feedback yields modest and variable gains in code repair and that stronger external feedback, including human-written feedback, yields larger ones \citep{olausson2024selfrepair}, while intrinsic self-correction can fail to improve or even degrade reasoning accuracy \citep{huang2024large}. Agent--computer-interface work documents cascading edit errors that simple guardrails interrupt \citep{yang2024sweagent}. At the system level, failure taxonomies for multi-agent systems \citep{cemri2025mast} and step-level attribution of agentic hallucinations \citep{liu2026agenthallubenchmarkingautomatedhallucination} classify where incidents originate and how they transmit, while \cite{jamshidi2026hallucination} find that hallucinations attenuate across the multi-agent cascades they evaluate, at the price of factual-information loss. Compounding errors along a trajectory have a theoretical antecedent in imitation learning \citep{ross2011reduction}. Experimental evidence also demonstrates propagation across connected
AI systems. \citet{cohen2025here} show that adversarial self-replicating
prompts can spread between AI-powered email assistants, causing
affected applications to compromise additional applications.
In multi-agent collaboration, \citet{xie2026spark} study six frameworks
and show that a single injected error can propagate through message
dependencies and produce widespread failure and persistent false
consensus. These studies establish or observe propagation, and some evaluate interventions such as feedback or guardrails; none jointly optimizes sustained delegation and costly mitigation effort under explicit incident-history excitation.

\paragraph{Agent safety evaluations and control protocols.}
Benchmarks show agents behaving unsafely with operational tools under benign instructions \citep{vijayvargiya2026openagentsafety} and executing harmful tasks when jailbroken \citep{andriushchenko2025agentharm}; the latter is an adversarial threat model outside our scope. Control protocols combine trusted monitoring with auditing \citep{greenblatt2024control} or screen actions by resampling \citep{bhatt2025ctrlz}, and scaling analyses ask how oversight fares as the overseen system grows more capable \citep{engels2025scaling}. These protocols can condition on history, but the level of oversight is a protocol parameter rather than a quantity optimized against observed incident dynamics. 

\paragraph{Human--AI delegation.}
Learning to defer decides, per instance, whether a model or a human should act \citep{madras2018predict,mozannar2020consistent}. Closer to our setting, \cite{fuchs2024delegation} train a manager that allocates control of a sequential task between human and autonomous agents, and \cite{lykouris2024learning} study deferral when human review decisions interact through congestion and delayed feedback. In these formulations errors and interventions can change the state, but none represents incident-history excitation explicitly or lets supervision effort suppress it; our problem adds both and lets the human choose how much effort to spend while holding control.

\paragraph{Cyber risk management.}
Risk induced by AI systems is increasingly examined with the concepts of established risk analysis \citep{thekdi2023disaster}, and its structure is particularly close to that of cyber risk management. The foundational article of \citet{gordon2002economics} posed the problem of optimal resource allocation under cyber threats, later extended to stochastic control \citep{callegaro2025stochastic,mastrolia2025agency}. At the enterprise level, quantification frameworks make this cost--benefit trade-off operational by coupling attack likelihood with the propagation of impact through functional
dependencies to the services and business processes that rely on compromised assets \citep{tatar2020quantification}. Contagion between cyber incidents is empirically documented \citep{baldwin2017contagion}, and Hawkes and contagion
models have entered the cyber-insurance literature relatively recently \citep{bessy2021multivariate,hernandez2026cyber,cherkaoui2026stress}, with
\citet{callegaro2025stochastic} bringing clustered attacks into the control problem. Note that these works do not consider turbulent event modeling in the volatility of the risky project, unlike this study.

\paragraph{Optimal switching and Hawkes-driven control.}
Impulse control has been applied to cyber-risk management by \cite{hillairet2026optimal} without incident clustering or a learning method. Relaxed (randomized) controls appear in continuous-time reinforcement learning in \cite{wang2020reinforcement} and in principal-agent contracting under moral hazard with project delegation in \cite{krvsek2023randomisation}, in both cases for continuous controls rather than regime switching. Relaxed switching, in which the controller sets the intensity of regime changes rather than their times, goes back to \cite{bouchard2009stochastic,elie2014bsde} and has been developed for policy-gradient learning by \cite{denkert2025control}.  What we add is the supervisory setting: two operating regimes with different controls and information, a frozen autonomous controller with its own objective, a project diffusion whose volatility depends on incident intensity. We extend the continuous-time reinforcement learning for controlled Hawkes jump-diffusions, with a control-dependent excitation kernel and Markovianization by exponential filter banks, developed by \cite{bielecki2026continuoustimereinforcementlearningcontrolled}to switching problem and turbulent AI-risks, see  Appendix \ref{app:markovian-approximation}, and a quadratic-variation criterion with switching costs. Note that the use of self-exciting activity to generate volatility has also been studied in \cite{horstxu2022} for financial models in a different framework; empirically fitted self-excitation among interacting automated agents in finance \citep{bacry2015hawkes,jaisson2015limit} provides an example of estimating this reduced form in a distinct application domain.

% ---------------------------------------------------------------
% Section 2
% ---------------------------------------------------------------
\begingroup
\section{From incident history to a delegation decision}
\label{sec:model}

\subsection{Running example: when should a supervisor take control?}
\label{sec:state}
Consider a recorded coding-agent session on Azure SDK for Java in AgentLogs \citep{richards2026AgentLogsDatasetOpening}. The agent attempts to read a script through an incorrect path, although an earlier search identifies its location; the read fails. Later, it reports that authenticated npm setup is missing, despite inspected source showing that setup before the installation command. Our contextual review labels these an action error and a state error, respectively. Appendix~\ref{app:history-data} gives the trace entries and timestamps. The sequence illustrates two incident types within one workflow.

This trace motivates a supervisory decision: after observing the failed read and its context, should a human take control, check subsequent claims more carefully, or let autonomous work continue? How much checking is worthwhile, and when should the agent regain autonomy? These hypothetical responses define the supervision problem studied below. It also raises a predictive question: does knowing earlier incidents improve forecasts of later errors? Section~\ref{sec:empirical-motivation} tests this predictive question across sessions.

We distinguish two operating regimes. In \emph{human-directed, AI-assisted operation}, the human chooses the productive action and how much effort to spend checking AI work; AI suggestions still enter the workflow, so incidents remain possible. In \emph{autonomous AI operation}, a fixed AI policy chooses the action, with only a small automatic safeguard level. The supervisory problem is therefore to choose productive work, checking effort, and the rate of handing control between these regimes over a horizon $[0,T]$.

\subsection{Why incident memory matters}
\label{sec:turbulenceHawkes}
Two coding projects can have the same current quality but different recent histories: one has proceeded without incidents, while the other has recently made an invalid tool call or an incorrect claim about its environment. Such history may contain information relevant to the next delegation decision, particularly if incorrect claims remain in the working context. We represent project quality and incident memory as distinct inputs to the delegation decision.

\begin{definition}[State and action incidents]
\label{def:hallucination-types}
A \emph{state incident} is an incorrect representation of the project or environment that the AI forms, reports, retrieves, or propagates; $N_t^S$ counts these incidents up to time $t$. An \emph{action incident} is an action the AI proposes, invokes, or reports that is inconsistent with its actual capabilities or the admissible action space; $N_t^A$ counts these incidents.
\end{definition}
In the AgentLogs example, the incorrect claim about npm setup is a state incident; the invalid script-read action is an action incident. Incident labels are defined relative to the project's admissible evidence and capabilities. This extends evidence-based factual assessment \citep{min2023factscore,es-etal-2024-ragas} to agent trajectories \citep{liu2026agenthallubenchmarkingautomatedhallucination}. We use ``turbulence'' to describe the resulting incident-related risk.

The \emph{intensity} $\lambda_t^i$ is the current event rate: over a short interval $dt$, the conditional expected number of type-$i$ incidents is approximately $\lambda_t^i dt$. A baseline $\mu_t^i$ allows incidents without prior errors. Past incidents add a history term: \emph{self-excitation} means that one type predicts more of itself, and \emph{cross-excitation} links the two types. A decaying kernel makes older incidents less influential. These reduced-form dependencies describe how predicted incident rates vary with observed history. The controlled Hawkes representation \citep{hawkes1971} is
\begin{equation}
 \lambda_t^i=\mu_t^i(\bar e_t)+
 \sum_{j\in\{S,A\}}\int_{(0,t)}\phi_{ij}(t-s,\bar e_t)\,dN_s^j,
 \qquad i\in\{S,A\}.
 \label{eq:main-intensity}
\end{equation}
Here $\phi_{ij}$ is the $(i,j)$ entry of the kernel matrix $\boldsymbol\Phi$: $j$ indexes a past incident and $i$ a future one. For example, $\phi_{SA}$ represents a possible increase in state-error risk after an action error. The entries are nonnegative and locally integrable; $\bar e_t$ is the effective checking effort defined below. The model assumes that relevant incident times and types are observable to the supervisor. Incident observation is an explicit information assumption of the controller.  We assume timely observation to isolate the role of incident history
in supervisory decisions. With detection delays, the supervisor
would instead act on the information available when incidents are
detected, potentially maintaining an estimate of unobserved incident
activity. Such delays can allow incidents to accumulate before
reactive takeover and may increase the value of sustained supervision.
The resulting policy depends on the delay relative to the persistence
of incident effects and on supervision costs. Incorporating this
mechanism requires a modified observation model; it is not included
in the present analysis.

\subsection{Empirical motivation: does history improve prediction?}
\label{sec:empirical-motivation}
To assess the predictive question raised by the running example, we analyze two archived datasets. \emph{AgentLogs}, the source of that example, records GitHub coding-agent execution traces; contextual review identifies action and state errors in 163 retained sessions from five repositories. We fit each repository independently in elapsed minutes and evaluate 78 held-out sessions containing 52 incidents. \emph{ODCV-Bench} records agents performing multi-step tasks with competing objectives and constraints \citep{li2025odcv}. Its 950 runs provide 9,017 scored decisions and 270 execution/protocol-failure decisions. Because per-step timestamps are unavailable, \textcolor{black}{this second analysis predicts binary failure occurrence on the decision clock.}

\begin{table}[t]
\centering\small
\caption{Held-out gain from incident history over a matched no-history baseline. Positive is better; nats are natural-log score units. AgentLogs uses $\rho=30$. Intervals are conditional, pointwise paired-bootstrap $90\%$ intervals.}
\label{tab:empirical-main}
\begin{tabular}{@{}lll@{}}
\toprule
Data and prediction score & Hawkes gain & $90\%$ interval\\
\midrule
AgentLogs: event timing (nats) & $10.043$ & $[3.016,17.611]$\\
AgentLogs: joint timing/type (nats) & $15.252$ & $[6.088,25.210]$\\
ODCV: binary occurrence (nats/1,000 decisions) & $4.227$ & $[1.894,7.044]$\\
\bottomrule
\end{tabular}
\end{table}

Both comparisons refit the same contextual baseline with and without exponential incident-history filters. AgentLogs uses a continuous-time intensity likelihood; its Poisson comparator includes repository/type-specific context and session age, \textcolor{black}{giving a context-dependent Poisson baseline.} The timing score rewards rate at observed events and penalizes rate over the full observation window; the joint score additionally rewards predicting incident type. ODCV uses a decision-time Bernoulli likelihood and matched model, scenario, prompt and decision-context covariates. Parameters remain fixed on held-out data; filters update using preceding incidents only.

Table~\ref{tab:empirical-main} shows that the Hawkes models achieve higher held-out log likelihood than the matched no-history baselines on both datasets, with positive gains whose $90\%$ intervals exclude zero. The analyses are exploratory: AgentLogs is selected, labels are provisional contextual annotations, occurrence times may precede detection, and the larger test split followed earlier specification exploration. Intervals omit that exploration and label uncertainty. \textcolor{black}{The empirical target is held-out incident prediction; the supervision mechanisms are specified in the control model below.} Appendix~\ref{app:empirical-history} gives definitions, estimation, split audits and limitations, including the distinct signed-filter working models used for prediction.

\subsection{From predictive memory to supervised operation}
\label{sec:relax}
The empirical result motivates retaining incident memory. Studying delegation requires additional assumptions about what supervision changes and how incidents affect project outcomes. The following control-model assumptions specify these mechanisms.

\paragraph{Project state and available decisions.}
Let $X_t\in\mathbb R$ summarize project quality, with larger values preferred, and let $I_t=0$ denote human-directed operation and $I_t=1$ autonomous AI operation. In the coding example, $X$ is an abstract measure of progress and correctness. In Human mode the controller chooses productive action $a_t\in A$ and checking effort $e_t\in E$, where $A,E$ are compact sets. The AI recommendation $\widehat a_t\in A$ comes from a policy held fixed throughout supervisory control. Effective action and effort are
\[
 \alpha_t=(1-I_{t-})a_t+I_{t-}\widehat a_t,
 \qquad \bar e_t=(1-I_{t-})e_t+I_{t-}\varepsilon,
\]
where $\varepsilon\in E$ is a small automatic safeguard level. We assume that checking weakens incident rates: $\mu^i$ and $\phi_{ij}$ are nonincreasing in effort. \textcolor{black}{Feedback experiments motivate this assumption \citep{olausson2024selfrepair}, with effectiveness depending on the feedback mechanism \citep{huang2024large}.} In Section~\ref{sec:numerical-ai}, effort suppresses only the history contribution, leaving baseline incidents possible even at maximal effort.

\paragraph{Three distinct consequences for project quality.}
Productive action changes expected progress through $b$. Background uncertainty $\sigma$ represents ordinary variability in execution, even without incident clustering. Incident-linked uncertainty $\gamma$ represents additional unpredictability during high-risk periods; unlike a direct loss, a diffusion shock can move quality in either direction. A state incident also causes a discrete loss of size $M^S$, such as rework prompted by an incorrect assessment of the build setup. We retain the dynamics
\begin{equation}
 dX_t=b(t,X_t,\alpha_t)dt+\sigma(t,X_t,\alpha_t)dW_t^\circ
 +\gamma(t,X_t,\bar e_t,\boldsymbol\lambda_t)dW_t^{\rm AI}
 -M_t^S dN_t^S.
 \label{eq:main-project}
\end{equation}
The Brownian drivers are independent; coefficients satisfy the regularity conditions in Appendix~\ref{app:ma-model}. Incident severity is exogenous and unaffected by checking. We assume $\gamma$ increases with incident intensity and does not increase with effort. These diffusion and severity assumptions specify how incident activity affects project outcomes. Action incidents instead corrupt the AI recommendation in Section~\ref{sec:blackbox}, allowing their effects to persist before reaching project quality.

\paragraph{Switching is controlled but stochastic.}
We use rate-based switching \citep{bouchard2009stochastic,denkert2025control}: the controller chooses the intensity of a switch, not its exact time. If $J_t$ counts switches,
\[
 dI_t=(1-2I_{t-})\,dJ_t,\qquad
 \nu_t=(1-I_{t-})\nu_t^{01}+I_{t-}\nu_t^{10},\qquad
 0\leq\nu_t^{ij}\leq\overline\nu<\infty.
\]
Thus $\nu^{10}$ controls takeover urgency and $\nu^{01}$ redelegation
urgency, allowing the policy to learn switching jointly with productive
action and checking effort. At a constant rate $\nu>0$, the waiting
time has mean $1/\nu$, and the probability of switching within $h$
is $1-e^{-\nu h}$. This models sustained supervisory allocation,
not a guaranteed instantaneous emergency stop. Critical incidents
may require an immediate override, outside the switching mechanism
studied here. Admissible controls $\pi=(a,e,\nu)\in\mathfrak U$ use available information without anticipation; formal definitions are in Appendix~\ref{app:model-formalities}.

\subsection{Objective and information available to the controller}
The supervisor trades project benefits against operation, checking, risk and handoff costs. With initial regime $i$ and the other initial conditions fixed, the objective is
\begin{equation}\label{eq:value}
 \begin{aligned}
 V(i)\!=\!\sup_{\pi\in\mathfrak U}\!\mathbb E^\pi\!\Bigg[\!
 g(X_T)\!+\!\int_0^T\!\left(\!f(s,X_s)\!-\!C_s\!-\!\frac{\kappa(I_{s-})}{2}\nu_s^2\!\right)ds\!-\!\frac{\eta}{2}[X]_T\!-\!\int_0^T\!\chi_{I_{s-},1-I_{s-}}dJ_s\!\Bigg].
 \end{aligned}
\end{equation}
Here $g$ rewards final quality and $f$ progress during the project. $C_s$ charges regime-dependent operating costs, including productive action and Human checking. The running term $\kappa\nu^2/2$ discourages persistently aggressive switching rates, while $\chi_{ij}>0$ is paid only when a handoff occurs. These are distinct costs. The risk term penalizes quadratic variation $[X]_T$: accumulated diffusion variation plus squared jump losses. It measures realized instability, combining continuous fluctuations with squared incident losses. The weight $\eta>0$ expresses the decision maker's risk preference.

Under the model, a recent cluster can justify continued checking even after current quality recovers; an event-free period can make redelegation attractive as memory decays. Whether this happens depends on assumed supervision effectiveness and costs. The operational controller observes time, project quality, current regime and recorded incident history. Our implemented policy also receives the raw AI recommendation in AI mode only, but not its latent error or true incident intensities. It chooses productive action and checking effort in Human mode and the outgoing switching rate in either mode. The next section explains how these observations become a policy.
\endgroup

\begingroup\color{black}
\section{Learning and evaluating a delegation policy}
\label{sec:numerical-ai}
The log analyses above evaluate incident prediction. Here we study control in a synthetic environment: \textcolor{black}{the simulator specifies the consequences of checking and delegation.} These experiments ask whether adaptive delegation improves the objective, whether memory filters help, and how assumed risk and costs change behavior.

\subsection{Hawkes-PPO: from observed incidents to decisions}
Hawkes-PPO combines an incident-history representation with PPO \citep{ppooriginal}. For each incident type $j$ and decay rate $\beta_q>0$, an exponential filter obeys
\begin{equation}
 dZ_t^{j,q}=-\beta_q Z_t^{j,q}dt+dN_t^j,\qquad Z_0^{j,q}=0.
 \label{eq:main-filter}
\end{equation}
An incident adds one to each filter of its type; between incidents the filters decay. Filters with different decay rates retain incident history over multiple timescales. Figure~\ref{fig:hawkes-ppo-method} shows how this memory enters the supervision loop.

\begin{figure}[t]
\centering\color{black}
\begin{tikzpicture}[
  x=1cm,y=1cm,>=Latex,
  box/.style={draw=blue!75,rounded corners=2pt,fill=blue!3,
    text=blue,align=center,text width=3.65cm,minimum height=.95cm,
    inner sep=4pt,font=\footnotesize},
  flow/.style={->,draw=blue!80,line width=.7pt},
  tag/.style={font=\scriptsize,text=blue,fill=white,inner sep=2pt}]
\node[box] (events) at (0,0) {\textbf{Observations}\\Project state $X$ + incidents};
\node[box] (memory) at (4.45,0) {\textbf{Incident memory filters}\\Multiple decay timescales};
\node[box] (policy) at (8.9,0) {\textbf{Hawkes-PPO policy}\\Memory + observations};
\node[box] (operation) at (0,-1.65) {\textbf{Project operation}\\Human or autonomous AI};
\node[box] (handoff) at (4.45,-1.65) {\textbf{Realized delegation}\\Takeover or redelegation};
\node[box] (decisions) at (8.9,-1.65) {\textbf{Chosen controls}\\Action $a$, effort $e$, rate $\nu$};
\draw[flow] (events) -- (memory);
\draw[flow] (events.north) -- ++(0,.45) -| node[tag,pos=.25] {project state $X$} (policy.north);
\draw[flow] (memory) -- (policy);
\draw[flow] (policy) -- (decisions);
\draw[flow] (decisions) -- node[tag,above] {$\nu$} (handoff);
\draw[flow] (handoff) -- node[tag,above] {$I$} (operation);
\draw[flow] (operation) -- (events);
\draw[flow] (decisions.south) -- ++(0,-.45) -| node[tag,pos=.25] {$a,e$ applied in Human mode} (operation.south);
\end{tikzpicture}
\caption{\textcolor{black}{The Hawkes-PPO supervision loop. Incident memory informs productive action, checking, and handoff urgency. Checking changes subsequent incident rates through the assumed supervision mechanism. The chosen rate controls the probability of a handoff; the realized transition changes the regime. PPO learns the policy from synthetic rollouts and objective rewards. The AI signal is observed only in AI mode.}}
\label{fig:hawkes-ppo-method}
\end{figure}

The simulator uses 20 rates per type, giving 40 memory coordinates. The policy receives neither true intensities nor excitation coefficients. It uses regime-specific networks; action and effort outputs apply only in Human mode. PPO updates the policy from sampled rewards and advantages, with clipped policy-ratio updates to limit each training step. The filters supply a structured memory of event timing; PPO supplies the reward-based optimization. \textcolor{black}{We evaluate the combined representation and optimization procedure through the comparisons below.} \textcolor{black}{Algorithm~\ref{alg:hawkes-ppo} gives the training and deployment procedure. Appendix~\ref{app:hawkes-ppo} supplies the losses, settings and detailed pseudocode (Algorithm~\ref{alg:hawkes-ppo-detailed}).} Exponential mixtures also give a finite-dimensional approximation of the mathematical control problem \citep{khabou2025markov,bielecki2026continuoustimereinforcementlearningcontrolled}. Under the stated assumptions, Theorem~\ref{thm:ma-full} controls the error in the \emph{full-information optimal value} as kernel approximation improves; Theorem~\ref{thm:ig-variance} bounds the additional loss from restricted observations under further assumptions. \textcolor{black}{These results bound kernel-approximation and information-restriction errors. Appendix~\ref{app:learned-policy-error} retains a separate neural-approximation and finite-training residual, while the experiments evaluate the fitted PPO policy.}

\begin{algorithm}[htbp]
\color{black}
\caption{Hawkes-PPO: learning delegation from incident history}
\label{alg:hawkes-ppo}
\small
\algrenewcommand{\algorithmicrequire}{\textbf{Input:}}
\algrenewcommand{\algorithmicensure}{\textbf{Output:}}
\begin{algorithmic}[1]
\Require Simulator, exponential decay rates $\{\beta_q\}$, and training settings in Table~\ref{tab:ppo-settings}.
\Ensure A policy for productive action $a$, checking effort $e$, and switching rate $\nu$.
\State Initialize separate actor and value networks for each current regime.
\State Assign half the environments to each initial regime; reset states and filters.
\For{each training rollout}
  \For{each decision in every environment}
    \State Observe $(t,X,I,Z,N^S,N^A,\mathbf1_{\{I=1\}}\widehat a^{\rm raw})$.
    \State Sample bounded $(a,e,\nu)$ from the actor for regime $I$.
    \State In Human mode apply $(a,e)$; in AI mode apply $(\widehat a,\varepsilon)$.
    \State Hold controls over the decision interval; realize switches at rate $\nu$.
    \State Decay filters between incidents and add one at each incident of their type, using \eqref{eq:main-filter}.
    \State Store the transition, full objective reward, action log probability and value prediction; reset episodes.
  \EndFor
  \State Compute and standardize advantages using the next-regime value; mask terminal transitions.
  \State Update actors with the clipped PPO objective and value networks with return-target regression over shuffled minibatches.
  \State Periodically evaluate both initial regimes and retain the checkpoint with the highest mean  objective.
\EndFor
\State Deploy the selected actors' transformed Gaussian means; handoffs remain stochastic at the chosen rate.
\end{algorithmic}
\end{algorithm}

\subsection{Synthetic environment and frozen AI recommendation}
\label{sec:blackbox}
We retain a scalar project with mean-reverting progress, action-dependent background noise, and saturating incident-linked noise:
\[
 b(t,x,a)=a-\delta x,\qquad \sigma(t,x,a)=\sigma_0+\sigma_1a,\qquad
 \gamma(t,x,e,\ell)=\gamma_0+\gamma_1\frac{\ell^S+\ell^A}{1+\ell^S+\ell^A}.
\]
Here $\ell$ denotes an intensity vector. The incident baseline is constant, $\mu_t^i(e)=\mu_i$, and the target kernel is
$\Phi_{ij}(r,e)=(1-e)A_{ij}(1+\rho_{ij}r)^{-\beta_{ij}}$,
represented numerically by a 20-term exponential mixture on the filter grid. \textcolor{black}{The power law specifies the simulator's incident memory.} Effort lowers excitation and hence incident-linked uncertainty indirectly through intensity; \textcolor{black}{the baseline risk and the displayed functional form of $\gamma$ remain fixed.} Operating costs are $c_0(a,e)=k_0+k_a a^2+k_e e^2$ and $c_1(a,e)=k_{\rm AI}+a^2$, with symmetric switching fee $\chi$. Table~\ref{tab:num-parameters} specifies all parameters and rewards.

\textcolor{black}{The synthetic autonomous AI is implemented as a frozen recommender.} Its nominal policy $a^*(t,x)$ is pretrained to optimize project benefits and background risk, excluding incident risk, supervision and switching (Appendix~\ref{app:nominal-ai}). We then add persistent recommendation error:
\[
 \begin{gathered}
 \widehat a_t^{\rm raw}=a^*(t,X_t)+R_t,\qquad
 \widehat a_t=\operatorname{clip}_{[0,1]}(\widehat a_t^{\rm raw}),\\
 dR_t=-\theta_0R_tdt+\theta_1s(L_t)dW_t^R-M^A dN_t^A,
 \end{gathered}
\]
where $L_t=\lambda_t^S+\lambda_t^A$. Unlike $X$, which measures project quality, $R$ measures error in the recommendation. Action incidents push it down; recovery pulls it toward zero, while higher incident intensity increases its uncertainty. It evolves in both regimes and is hidden from the policy. \textcolor{black}{This assumed error process specifies persistence and recovery in corrupted recommendations.}

\subsection{Does adaptive delegation improve the objective?}
We evaluate each selected policy on 4,096 fresh common-noise paths per initial regime. Table~\ref{tab:control-main} compares the same objective, including all risk and cost terms, from matching starts. Hawkes-PPO improves the point estimate over fixed AI by $4.519$ from an AI start and over fixed Human by $1.351$ from a Human start. \textcolor{black}{The approximate full-information HJB/DGM benchmark has a higher evaluated value. Intervals quantify Monte Carlo evaluation uncertainty for the selected policies, with training seeds and fitted weights held fixed.}

\begin{table}[t]
\color{black}\centering\small
\caption{Synthetic control objective, with $90\%$ Monte Carlo intervals; higher is better. Fixed regimes have only their matching start. All RL methods and implementation details appear in Table~\ref{app:all-values}.}
\label{tab:control-main}
\begin{tabular}{@{}lcc@{}}
\toprule
Policy & AI start & Human start\\
\midrule
HJB/DGM benchmark & $-93.262\;[-93.508,-93.015]$ & $-94.230\;[-94.468,-93.991]$\\
Hawkes-PPO & $-93.608\;[-93.861,-93.354]$ & $-94.676\;[-94.915,-94.437]$\\
PPO without filters & $-94.954\;[-95.198,-94.710]$ & $-97.219\;[-97.407,-97.031]$\\
Fixed AI & $-98.127\;[-98.537,-97.718]$ & ---\\
Fixed Human (HJB/DGM) & --- & $-96.027\;[-96.191,-95.864]$\\
\bottomrule
\end{tabular}
\end{table}

Figure~\ref{fig:comparison} additionally reports terminal quality and mean pathwise volatility $\mathbb E[\sqrt{[X]_T/T}]$. \textcolor{black}{The delegation policy jointly values project quality, risk reduction and Human operating cost.} Figure~\ref{fig:path} illustrates one simulated trajectory: after incidents during AI operation, takeover intensity rises; during event-free Human spells, checking falls and redelegation intensity rises as history decays. The realized switch times remain stochastic.

\begin{figure}[htbp]
 \color{black}\centering\includegraphics[width=\linewidth]{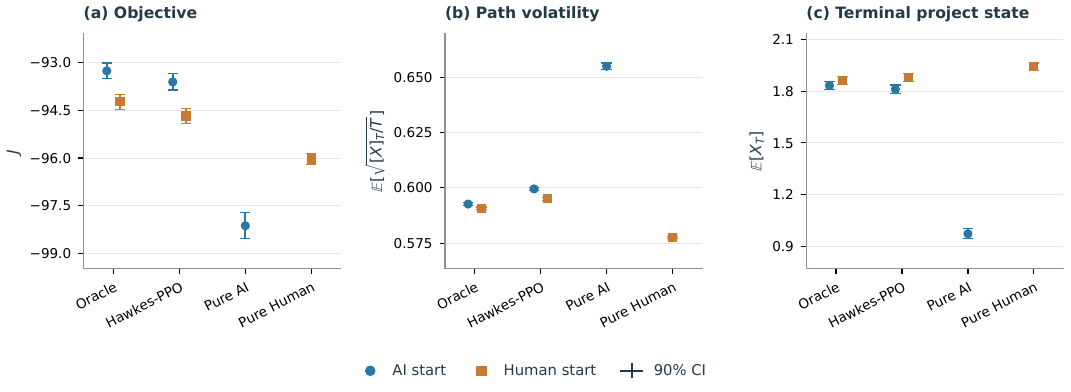}
 \caption{\footnotesize Four-method comparison. Error bars are \(90\%\) confidence intervals.
 Higher objective and terminal state are better; lower realized volatility
 indicates less quadratic variation.}
 \label{fig:comparison}
\end{figure}
\begin{figure}[htbp]
 \color{black}\centering\includegraphics[width=\linewidth]{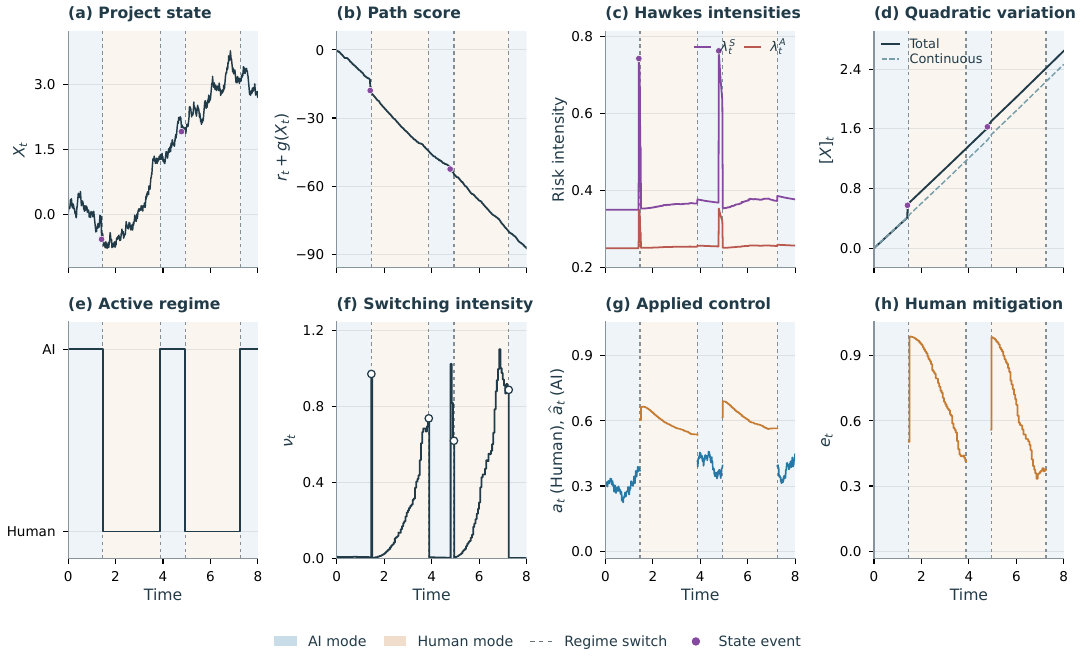}
 \caption{\footnotesize Vertical dashed lines mark regime switches. From left to right, top to bottom: paths of $X$, net running reward and terminal reward, intensity of state and action turbulence (Hawkes intensities), quadratic variation of the project, regime activation, controlled switching intensity, action of the agent human or AI, human effort.}
 \label{fig:path}
\end{figure}

\subsection{Do incident-history filters help?}
The PPO comparison holds network dimensions and training protocol fixed, setting filter inputs to zero in the unfiltered policy. Both still observe cumulative incident counts, \textcolor{black}{so the comparison isolates the added timing representation.} Filters improve PPO's evaluated objective from both starts. Among the eight RL variants in Appendix~\ref{app:other-rl}, Hawkes-PPO alone exceeds fixed Human's point estimate from a Human start. The same filters do not improve SAC or DDPG under this budget. \textcolor{black}{The result supports this implemented combination of representation and optimization under the stated training budget.}

\subsection{How do risk and costs change supervision?}
\label{sec:supervision-statics}
We vary excitation strength, Human fixed operating cost, and per-switch fees while retaining the same policy inputs and training protocol (Appendix~\ref{app:supervision-statics}). Stronger excitation produces more time under Human control: from an AI start, Human time is $0.9\%$, $46.4\%$, and $73.6\%$ at $0.5$, $1$, and $1.5$ times baseline excitation. \textcolor{black}{Checking effort per unit of Human time varies across numerical configurations.} Increasing Human fixed cost from $0.75$ to $1.25$ times baseline reduces Human time from $76.0\%$ to $13.7\%$ for AI starts and from $99.6\%$ to $19.1\%$ for Human starts. At baseline excitation, conditional checking effort rises, incidents and bursts increase, and terminal quality falls. Crossed checks preserve the Human-time and incident directions, but effort and quality have exceptions. Raising the switch fee reduces switching from both starts and increases persistence in the initial regime: Human time falls for AI starts but rises for Human starts. Table~\ref{tab:comparative-statics} retains all outcomes; Table~\ref{tab:comparative-intervals} reports their uncertainty. \textcolor{black}{These responses characterize the policy under the specified supervision mechanisms and costs.}
\endgroup

\section{Conclusion}
\label{sec:conclusion}
We introduced a framework for turbulent AI risk in which AI-generated incidents can trigger further incidents and amplify uncertainty in a project’s evolution. Within this framework, we formulated an optimal switching problem between full delegation to a black-box AI system and human–AI assisted decision-making. The analysis captures the trade-off between the lower operating costs of AI delegation and the cost of human intervention to mitigate cascading risks. It highlights the importance of adapting human involvement to the evolving risk state, accounting for both immediate operating costs and the consequences of incident accumulation. By modeling cascading incidents and adaptive human intervention, this work contributes to the broader study of AI risk management \cite{ziosi2026open}, with a specific focus on the dynamic accumulation and amplification of AI-generated risks. 

\textcolor{black}{Our archived-data analyses support useful incident-history information beyond specified context baselines.  Deployment would additionally require calibrated supervision effects, reliable incident observation, and validation of project dynamics and costs.} A natural extension is to consider interconnected AI agents cooperating under human oversight. Whether incidents attenuate or amplify across agents is empirically open. \citep{jamshidi2026hallucination} find attenuation but incidents could propagate across agents and create feedback loops, turning local failures into systemic risks. This extension would raise a further control problem: how to allocate human oversight across a network to contain cascading incidents while retaining the benefits of AI delegation.

\textcolor{black}{Finally, a natural extension of this work would be to incorporate detection delays, so that the controller can respond to an incident only after it has been detected. Explicitly modeling detection policies would then be important, as delayed responses may significantly affect the evolution of the project state and the effectiveness of intervention. We leave this extension for future work.}

% Compact table: separate AI-start and Human-start panels.

\label{main:body-end}
% ---------------------------------------------------------------
% Section 4
% ---------------------------------------------------------------

% References
% ---------------------------------------------------------------
\newpage 

\bibliography{ref}

@article{hawkes1971,
  author  = {Hawkes, A. G.},
  title   = {Spectra of Some Self-Exciting and Mutually Exciting Point Processes},
  journal = {Biometrika},
  volume  = {58},
  number  = {1},
  pages   = {83--90},
  year    = {1971},
  doi     = {10.1093/biomet/58.1.83}
}

@inproceedings{cohen2025here,
  title={Here comes the AI worm: Preventing the propagation of adversarial self-replicating prompts within GenAI ecosystems},
  author={Cohen, Stav and Bitton, Ron and Nassi, Ben},
  booktitle={Proceedings of the 2025 ACM SIGSAC Conference on Computer and Communications Security},
  pages={3975--3989},
  year={2025}
}

@article{xie2026spark,
  title={From spark to fire: Modeling and mitigating error cascades in LLM-based multi-agent collaboration},
  author={Xie, Yizhe and Zhu, Congcong and Zhang, Xinyue and Zhu, Tianqing and Ye, Dayong and Qi, Minfeng and Chen, Huajie and Zhou, Wanlei},
  journal={arXiv preprint arXiv:2603.04474},
  year={2026}
}

@article{horstxu2022,
  author  = {Horst, Ulrich and Xu, Wei},
  title   = {The Microstructure of Stochastic Volatility Models with Self-Exciting Jump Dynamics},
  journal = {The Annals of Applied Probability},
  volume  = {32},
  number  = {6},
  pages   = {4568--4610},
  year    = {2022},
  doi     = {10.1214/22-AAP1796}
}

@inproceedings{zhang2024snowball,
  author    = {Zhang, Muru and Press, Ofir and Merrill, William and Liu, Alisa and Smith, Noah A.},
  title     = {How Language Model Hallucinations Can Snowball},
  booktitle = {Proceedings of the 41st International Conference on Machine Learning},
  series    = {Proceedings of Machine Learning Research},
  volume    = {235},
  pages     = {59670--59684},
  year      = {2024}
}

@misc{liu2026agenthallubenchmarkingautomatedhallucination,
      title={AgentHallu: Benchmarking Automated Hallucination Attribution of {LLM}-based Agents}, 
      author={Xuannan Liu and Xiao Yang and Zekun Li and Peipei Li and Ran He},
      year={2026},
      eprint={2601.06818},
      archivePrefix={arXiv},
      primaryClass={cs.CL},
      url={https://arxiv.org/abs/2601.06818}, 
}

@article{jamshidi2026hallucination,
  title={Hallucination Cascade: Analyzing Error Propagation in Multi-Agent LLM Systems},
  author={Jamshidi, Saeid and Dakhel, Arghavan Moradi and Nafi, Kawser Wazed and Khomh, Foutse},
  journal={arXiv preprint arXiv:2606.07937},
  year={2026}
}

@article{bouchard2009stochastic,
  title={A stochastic target formulation for optimal switching problems in finite horizon},
  author={Bouchard, Bruno},
  journal={Stochastics: An International Journal of Probability and Stochastic Processes},
  volume={81},
  number={2},
  pages={171--197},
  year={2009},
  publisher={Taylor \& Francis}
}

@article{denkert2025control,
  title={Control randomisation approach for policy gradient and application to reinforcement learning in optimal switching},
  author={Denkert, Robert and Pham, Huy{\^e}n and Warin, Xavier},
  journal={Applied Mathematics \& Optimization},
  volume={91},
  number={1},
  pages={9},
  year={2025},
  publisher={Springer}
}

@article{elie2014bsde,
  title={{BSDE} representations for optimal switching problems with controlled volatility},
  author={Elie, Romuald and Kharroubi, Idris},
  journal={Stochastics and Dynamics},
  volume={14},
  number={03},
  pages={1450003},
  year={2014},
  publisher={World Scientific}
}

@article{wang2020reinforcement,
  title={Reinforcement learning in continuous time and space: A stochastic control approach},
  author={Wang, Haoran and Zariphopoulou, Thaleia and Zhou, Xun Yu},
  journal={Journal of Machine Learning Research},
  volume={21},
  number={198},
  pages={1--34},
  year={2020}
}

@article{krvsek2023randomisation,
  title={Randomisation with moral hazard: a path to existence of optimal contracts},
  author={Kr{\v{s}}ek, Daniel and Possama{\"\i}, Dylan},
  journal={arXiv preprint arXiv:2311.13278},
  year={2023}
}

@misc{bielecki2026continuoustimereinforcementlearningcontrolled,
      title={Continuous-Time Reinforcement Learning for Controlled Hawkes Jump-Diffusions}, 
      author={Tomasz R. Bielecki and Thibaut Mastrolia and Haoze Yan},
      year={2026},
      eprint={2608.19151},
      archivePrefix={arXiv},
      primaryClass={cs.LG},
      url={https://arxiv.org/abs/2608.19151}, 
}

@inproceedings{min2023factscore,
  author    = {Min, Sewon and Krishna, Kalpesh and Lyu, Xinxi and Lewis, Mike and Yih, Wen-tau and Koh, Pang Wei and Iyyer, Mohit and Zettlemoyer, Luke and Hajishirzi, Hannaneh},
  title     = {{FActScore}: Fine-Grained Atomic Evaluation of Factual Precision in Long Form Text Generation},
  booktitle = {Proceedings of the 2023 Conference on Empirical Methods in Natural Language Processing},
  pages     = {12076--12100},
  year      = {2023}
}

@inproceedings{es-etal-2024-ragas,
  title={Ragas: Automated evaluation of retrieval augmented generation},
  author={Es, Shahul and James, Jithin and Anke, Luis Espinosa and Schockaert, Steven},
  booktitle={Proceedings of the 18th conference of the european chapter of the association for computational linguistics: system demonstrations},
  pages={150--158},
  year={2024}
}

@misc{cheng2026deterministicpolicygradientreinforcement,
      title={Deterministic Policy Gradient for Reinforcement Learning with Continuous Time and State}, 
      author={Ziheng Cheng and Xin Guo and Yufei Zhang},
      year={2026 (v2)},
      eprint={2509.23711},
      archivePrefix={arXiv},
      primaryClass={cs.LG},
      url={https://arxiv.org/abs/2509.23711}, 
}

@misc{ppooriginal,
  author = {Schulman, J. and Wolski, F. and Dhariwal, P.
            and Radford, A. and Klimov, O.},
  title = {Proximal Policy Optimization Algorithms},
  year = {2017},
  eprint = {1707.06347},
  archivePrefix = {arXiv},
  url = {https://arxiv.org/abs/1707.06347}
}

@article{dgm,
  author = {Sirignano, J. and Spiliopoulos, K.},
  title = {{DGM}: A Deep Learning Algorithm for Solving
           Partial Differential Equations},
  journal = {Journal of Computational Physics},
  volume = {375},
  pages = {1339--1364},
  year = {2018},
  doi = {10.1016/j.jcp.2018.08.029}
}

@inproceedings{haarnoja2018sac,
  title={Soft actor-critic: Off-policy maximum entropy deep reinforcement learning with a stochastic actor},
  author={Haarnoja, Tuomas and Zhou, Aurick and Abbeel, Pieter and Levine, Sergey},
  booktitle={International conference on machine learning},
  pages={1861--1870},
  year={2018},
  organization={Pmlr}
}

@article{lillicrap2015continuous,
  title={Continuous control with deep reinforcement learning},
  author={Lillicrap, Timothy P and Hunt, Jonathan J and Pritzel, Alexander and Heess, Nicolas and Erez, Tom and Tassa, Yuval and Silver, David and Wierstra, Daan},
  journal={arXiv preprint arXiv:1509.02971},
  year={2015}
}

@article{hao2024quantifying,
  title={Quantifying the uncertainty of {LLM} hallucination spreading in complex adaptive social networks},
  author={Hao, Guozhi and Wu, Jun and Pan, Qianqian and Morello, Rosario},
  journal={Scientific reports},
  volume={14},
  number={1},
  pages={16375},
  year={2024},
  publisher={Nature Publishing Group UK London}
}

@inproceedings{sinha2026illusion,
  author    = {Akshit Sinha and Arvindh Arun and Shashwat Goel and Steffen Staab and Jonas Geiping},
  title     = {The illusion of diminishing returns: Measuring long horizon execution in {LLMs}},
  booktitle = {International Conference on Learning Representations},
  year      = {2026},
  url       = {https://proceedings.iclr.cc/paper_files/paper/2026/file/3b4e1336f775c3dba16ebbb8d2afd258-Paper-Conference.pdf},
  note      = {arXiv:2509.09677}
}

@inproceedings{yang2024sweagent,
  title     = {{SWE}-agent: Agent-Computer Interfaces Enable Automated Software Engineering},
  author    = {Yang, John and Jimenez, Carlos E. and Wettig, Alexander and Lieret, Kilian and Yao, Shunyu and Narasimhan, Karthik and Press, Ofir},
  booktitle = {Advances in Neural Information Processing Systems},
  year      = {2024},
  note      = {arXiv:2405.15793},
  url       = {https://arxiv.org/abs/2405.15793}
}

@inproceedings{cemri2025mast,
  title     = {Why Do Multi-Agent {LLM} Systems Fail?},
  author    = {Cemri, Mert and Pan, Melissa Z. and Yang, Shuyi and Agrawal, Lakshya A. and Chopra, Bhavya and Tiwari, Rishabh and Keutzer, Kurt and Parameswaran, Aditya and Klein, Dan and Ramchandran, Kannan and Zaharia, Matei and Gonzalez, Joseph E. and Stoica, Ion},
  booktitle = {Advances in Neural Information Processing Systems},
  year      = {2025},
  note      = {arXiv:2503.13657},
  url       = {https://arxiv.org/abs/2503.13657}
}

@inproceedings{olausson2024selfrepair,
  title     = {Is Self-Repair a Silver Bullet for Code Generation?},
  author    = {Olausson, Theo X. and Inala, Jeevana Priya and Wang, Chenglong and Gao, Jianfeng and Solar-Lezama, Armando},
  booktitle = {International Conference on Learning Representations},
  year      = {2024},
  note      = {arXiv:2306.09896},
  url       = {https://arxiv.org/abs/2306.09896}
}

@inproceedings{huang2024large,
  title={Large language models cannot self-correct reasoning yet},
  author={Huang, Jie and Chen, Xinyun and Mishra, Swaroop and Zheng, Huaixiu Steven and Yu, Adams and Song, Xinying and Zhou, Denny},
  booktitle={International conference on learning representations},
  volume={2024},
  year={2024}
}

@article{engels2025scaling,
  title={Scaling laws for scalable oversight},
  author={Engels, Joshua and Baek, David and Kantamneni, Subhash and Tegmark, Max},
  journal={Advances in Neural Information Processing Systems},
  volume={38},
  year={2025}
}

@inproceedings{greenblatt2024control,
  title     = {{AI} Control: Improving Safety Despite Intentional Subversion},
  author    = {Greenblatt, Ryan and Shlegeris, Buck and Sachan, Kshitij and Roger, Fabien},
  booktitle = {International Conference on Machine Learning},
  year      = {2024},
  pages={16295 - 16336}
}

@article{bhatt2025ctrlz,
  title   = {Ctrl-{Z}: Controlling {AI} Agents via Resampling},
  author  = {Bhatt, Aryan and Rushing, Cody and Kaufman, Adam and Tracy, Tyler and Georgiev, Vasil and Matolcsi, David and Khan, Akbir and Shlegeris, Buck},
  journal = {arXiv preprint arXiv:2504.10374},
  year    = {2025},
  url     = {https://arxiv.org/abs/2504.10374}
}

@inproceedings{andriushchenko2025agentharm,
  title     = {{AgentHarm}: A Benchmark for Measuring Harmfulness of {LLM} Agents},
  author    = {Andriushchenko, Maksym and Souly, Alexandra and Dziemian, Mateusz and Duenas, Derek and Lin, Maxwell and Wang, Justin and Hendrycks, Dan and Zou, Andy and Kolter, Zico and Fredrikson, Matt and Winsor, Eric and Wynne, Jerome and Gal, Yarin and Davies, Xander},
  booktitle = {International Conference on Learning Representations},
  year      = {2025},
  note      = {arXiv:2410.09024},
  url       = {https://arxiv.org/abs/2410.09024}
}

@inproceedings{vijayvargiya2026openagentsafety,
  title     = {{OpenAgentSafety}: A Comprehensive Framework for Evaluating Real-World {AI} Agent Safety},
  author    = {Vijayvargiya, Sanidhya and Soni, Aditya Bharat and Zhou, Xuhui and Wang, Zora Zhiruo and Dziri, Nouha and Neubig, Graham and Sap, Maarten},
  booktitle = {International Conference on Learning Representations},
  year      = {2026},
  note      = {arXiv:2507.06134},
  url       = {https://arxiv.org/abs/2507.06134}
}

@misc{anthropic2025vend,
  title        = {Project Vend: Can {Claude} Run a Small Shop? ({A}nd Why Does That Matter?)},
  author       = {{Anthropic}},
  howpublished = {\url{https://www.anthropic.com/research/project-vend-1}},
  year         = {2025},
  month        = jun
}

@inproceedings{madras2018predict,
  title     = {Predict Responsibly: Improving Fairness and Accuracy by Learning to Defer},
  author    = {Madras, David and Pitassi, Toniann and Zemel, Richard},
  booktitle = {Advances in Neural Information Processing Systems},
  year      = {2018}
}

@inproceedings{mozannar2020consistent,
  title     = {Consistent Estimators for Learning to Defer to an Expert},
  author    = {Mozannar, Hussein and Sontag, David},
  booktitle = {International Conference on Machine Learning},
  year      = {2020},
  url={https://arxiv.org/abs/2006.01862}
}

@article{fuchs2024delegation,
  title={Optimizing delegation in collaborative human-AI hybrid teams},
  author={Fuchs, Andrew and Passarella, Andrea and Conti, Marco},
  journal={ACM Transactions on Autonomous and Adaptive Systems},
  volume={19},
  number={4},
  pages={1--33},
  year={2024},
  publisher={ACM New York, NY}
}

@article{lykouris2024learning,
  title={Learning to defer in congested systems: the {AI}-human interplay},
  author={Lykouris, Thodoris and Weng, Wentao},
  journal={arXiv preprint arXiv:2402.12237},
  year={2024}
}

@inproceedings{ross2011reduction,
  title     = {A Reduction of Imitation Learning and Structured Prediction to No-Regret Online Learning},
  author    = {Ross, St{\'e}phane and Gordon, Geoffrey J. and Bagnell, J. Andrew},
  booktitle = {Proceedings of the 14th International Conference on Artificial Intelligence and Statistics},
  year      = {2011},
  note      = {15:627-635}
}

@article{bacry2015hawkes,
  title   = {Hawkes Processes in Finance},
  author  = {Bacry, Emmanuel and Mastromatteo, Iacopo and Muzy, Jean-Fran{\c{c}}ois},
  journal = {Market Microstructure and Liquidity},
  volume  = {1},
  number  = {1},
  pages   = {1550005},
  year    = {2015},
  note    = {arXiv:1502.04592}
}

@article{kammler1976,
title = {Approximation with sums of exponentials in Lp[0, $\infty$)},
journal = {Journal of Approximation Theory},
volume = {16},
number = {4},
pages = {384-408},
year = {1976},
issn = {0021-9045},
doi = {https://doi.org/10.1016/0021-9045(76)90071-X},
url = {https://www.sciencedirect.com/science/article/pii/002190457690071X},
author = {David W Kammler}
}

@article{mastrolia2025agency,
  title={Agency problems and adversarial bilevel optimization under uncertainty and cyber threats},
  author={Mastrolia, Thibaut and Yan, Haoze},
  journal={arXiv preprint arXiv:2505.08989},
  year={2025}
}

@article{khabou2025markov,
  title={{M}arkov approximation for controlled {H}awkes Jump-Diffusions with general kernels},
  author={Khabou, Mahmoud and Talbi, Mehdi},
  journal={arXiv preprint arXiv:2507.11294},
  year={2025}
}

@article{hillairet2026optimal,
  title={Optimal impulse control for cyber risk management},
  author={Hillairet, Caroline and Mastrolia, Thibaut and Sabbagh, Wissal},
  journal={Applied Mathematics \& Optimization},
  volume={94},
  number={1},
  pages={18},
  year={2026},
  publisher={Springer}
}

@article{hernandez2026cyber,
  title={Cyber risk prevention under risk averse spectral criteria},
  author={Hern{\'a}ndez-Santib{\'a}{\~n}ez, Nicol{\'a}s and Kazi-Tani, Nabil and Vazquez-Gaete, Mariano},
  journal={Applied Mathematics \& Optimization},
  volume={94},
  number={1},
  pages={14},
  year={2026},
  publisher={Springer}
}

@techreport{cherkaoui2026stress,
  author      = {Yousra Cherkaoui and Caroline Hillairet and Thomas Peyrat and Anthony R{\'e}veillac},
  title       = {Stress scenarios of cyber loss processes with dependencies},
  type        = {Working paper},
  number      = {hal-05558990},
  institution = {HAL},
  year        = {2026},
  url         = {https://hal.science/hal-05558990v1}
}

@article{bessy2021multivariate,
  title={Multivariate {H}awkes process for cyber insurance},
  author={Bessy-Roland, Yannick and Boumezoued, Alexandre and Hillairet, Caroline},
  journal={Annals of Actuarial Science},
  volume={15},
  number={1},
  pages={14--39},
  year={2021},
  publisher={Cambridge University Press}
}

@article{callegaro2025stochastic,
  title={A stochastic Gordon--Loeb model for optimal cybersecurity investment under clustered attacks},
  author={Callegaro, Giorgia and Fontana, Claudio and Hillairet, Caroline and Ongarato, Beatrice},
  journal={Annals of Actuarial Science},
  pages={1--29},
  year={2026},
  publisher={Cambridge University Press}
}

@article{gordon2002economics,
  title={The economics of information security investment},
  author={Gordon, Lawrence A and Loeb, Martin P},
  journal={ACM Transactions on Information and System Security (TISSEC)},
  volume={5},
  number={4},
  pages={438--457},
  year={2002},
  publisher={ACM New York, NY, USA}
}

@article{jaisson2015limit,
  title={Limit theorems for nearly unstable {H}awkes processes},
  author={Jaisson, Thibault and Rosenbaum, Mathieu},
  journal={The Annals of Applied Probability},
  volume={25},
  number={2},
  year={2015}
}

@article{bremaud1996,
 ISSN = {00911798, 2168894X},
 URL = {http://www.jstor.org/stable/2244985},
 author = {Pierre Brémaud and Laurent Massoulié},
 journal = {The Annals of Probability},
 number = {3},
 pages = {1563--1588},
 publisher = {Institute of Mathematical Statistics},
 title = {Stability of Nonlinear Hawkes Processes},
 urldate = {2026-09-21},
 volume = {24},
 year = {1996}
}

@techreport{tatar2020quantification,
  author      = {Tatar, Unal and Keskin, Omer and
                 Bahsi, Hayretdin and Pinto, C. Ariel},
  title       = {Quantification of Cyber Risk for Actuaries:
                 An Economic-Functional Approach},
  institution = {Society of Actuaries},
  type        = {Research Report},
  year        = {2020},
  month       = may,
  url         = {https://www.soa.org/resources/research-reports/2020/quantification-cyber-risk/}
}

@article{thekdi2023disaster,
  author  = {Thekdi, Shital and Tatar, Unal and
             Santos, Joost and Chatterjee, Samrat},
  title   = {Disaster risk and artificial intelligence:
             A framework to characterize conceptual synergies
             and future opportunities},
  journal = {Risk Analysis},
  year    = {2023},
  volume  = {43},
  number  = {8},
  pages   = {1641--1656},
  doi     = {10.1111/risa.14038}
}

@article{baldwin2017contagion,
  author  = {Baldwin, Adrian and Gheyas, Iffat and Ioannidis, Christos and
             Pym, David and Williams, Julian},
  title   = {Contagion in cyber security attacks},
  journal = {Journal of the Operational Research Society},
  year    = {2017},
  volume  = {68},
  number  = {7},
  pages   = {780--791},
  doi     = {10.1057/jors.2016.37}
}

@article{ziosi2026open,
  title={Open Problems in Frontier AI Risk Management},
  author={Ziosi, Marta and Plueckebaum, Miro and Casper, Stephen and Papadatos, Henry and Chin, Ze Shen and Slattery, Peter and Gealy, James and Rudner, Tim GJ and Tse, Brian and Gil, Ariel and others},
  journal={arXiv preprint arXiv:2604.25982},
  year={2026}
}

@misc{li2025odcv,
  title = {A Benchmark for Evaluating Outcome-Driven Constraint Violations in Autonomous AI Agents},
  author = {Li, Miles Q. and Fung, Benjamin C. M. and Weiss, Martin and Xiong, Pulei and Al-Hussaeni, Khalil and Fachkha, Claude},
  year = {2025},
  eprint = {2512.20798},
  archivePrefix = {arXiv},
  primaryClass = {cs.AI},
  url = {https://arxiv.org/abs/2512.20798},
  note = {Archived transcripts analyzed at repository commit 6924dafb3a3f84fc320fbe9f0a08203d67d1d488}
}

@misc{richards2026AgentLogsDatasetOpening,
  title={AgentLogs: A Dataset for Opening the Black Box of Cloud Agents on GitHub},
  author={Richards, Jonan and Horikawa, Kosei and Fan, Youmei and Kashiwa, Yutaro and Wessel, Mairieli},
  year={2026},
  eprint={2608.29204},
  archivePrefix={arXiv},
  primaryClass={cs.SE},
  doi={10.48550/arXiv.2608.29204},
  url={https://arxiv.org/abs/2608.29204}
}

@article{shlomovich2022binned,
 author={Shlomovich, Leigh and Cohen, Edward A. K. and Adams, Niall and Patel, Lekha},
 title={Parameter Estimation of Binned {Hawkes} Processes},
 journal={Journal of Computational and Graphical Statistics}, volume={31}, number={4},
 pages={990--1000}, year={2022}, doi={10.1080/10618600.2022.2050247},
 url={https://doi.org/10.1080/10618600.2022.2050247}}

@article{rizoiu2022,
 author={Rizoiu, Marian-Andrei and Soen, Alexander and Li, Shidi and Calderon, Pio and Dong, Leanne J. and Menon, Aditya Krishna and Xie, Lexing},
 title={Interval-censored {Hawkes} Processes}, journal={Journal of Machine Learning Research},
 volume={23}, number={338}, pages={1--84}, year={2022},
 url={https://www.jmlr.org/papers/v23/21-0917.html}}

@article{filimonov2015,
 author={Filimonov, Vladimir and Sornette, Didier},
 title={Apparent Criticality and Calibration Issues in the {Hawkes} Self-Excited Point Process Model: Application to High-Frequency Financial Data},
 journal={Quantitative Finance}, volume={15}, number={8}, pages={1293--1314}, year={2015},
 doi={10.1080/14697688.2015.1032544}, url={https://arxiv.org/abs/1308.6756}}
\bibliographystyle{plainnat}

% ---------------------------------------------------------------
% Appendix
% ---------------------------------------------------------------
\appendix

% Replacement for the Appendix A heading and subsections A.1--A.3 only.
% First replace the theorem declarations using appendix_A_theorem_styles.tex.
% Red: requested changes 1--4. Notation-only changes have no added color.
% All assumption bodies are red because their font changes to italic.
% The new remark is unnumbered to retain the existing A.1--A.6 numbering.
% A.4 is not included. The original manuscript is unchanged.

\section{Markovian approximation and observable-policy stability}
\label{app:markovian-approximation}

We use the approximation framework of
\cite{khabou2025markov,bielecki2026continuoustimereinforcementlearningcontrolled}, with mixture of exponential approximation
background from \cite{kammler1976}. We extend their framework to
intensity-dependent diffusion, especially with the dependency $\mathbf{\lambda}$ in the volatility process $\gamma$. The main difficulty is in the observation of $\hat a$. We will distinguish a full observation mode in which the error process $R$ is fully observed and only the recommendation $G$ is observed. 

We first state the main assumptions on the coefficient in our model, see Section \ref{app:ma-model}. Then, we introduce the definition of memory filters and the associated convergences of the approached intensity, derived from the mixture of exponential to the primal one, see Section \ref{app:ma-inherited}. We then turn to the convergence of the approached problem to the value function of the primal optimization \eqref{eq:value} when the error $R$ induced by the AI is observed, see Section \ref{app:ma-full}. Finally, in Section \ref{app:ma-observable}, when the recommendation $G$ is observed in AI mode but not the error $R$, we derive information-loss bounds between full-information and observable-history control for a fixed recommendation map \(G\), under additional verification, sensitivity, and policy-existence assumptions.

\subsection{Model and assumptions}
\label{app:ma-model}
We use the same framework as in \cite{bielecki2026continuoustimereinforcementlearningcontrolled} with the similar Poisson embedding construction in \cite{bremaud1996}. Write $\pi=(a,e,\nu)\in \mathfrak U$. We recall
\[
 \begin{gathered}
 \mathbf{S}_t=(X_t,R_t),\qquad
 \widehat a_t=\operatorname{Proj}_A G(t,\mathbf{S}_{t-}),\\
 \alpha_t:=\alpha(t,\mathbf s,i,\pi)=(1-I_{t-})a_t+I_{t-}\widehat a_t,\qquad
 \bar e_t:=\bar e(i,\pi)=(1-I_{t-})e_t+I_{t-}\varepsilon.
 \end{gathered}
\]

Use Euclidean vector norms and the Frobenius norm for diffusion matrices.
For intensity and count vectors, write $\|v\|_1:=\sum_i|v_i|$.
For kernel matrices $M\in\mathbb R^{2\times2}$, use the induced matrix norm
\[
 \|M\|:=\|M\|_{1\to1}
 :=\sup_{v\in\mathbb R^2\setminus\{0\}}
       \frac{\|Mv\|_1}{\|v\|_1}
 =\max_{j=1,2}\sum_{i=1}^2|M_{ij}|.
\] For each integer $K\ge1$, we define $\Phi^K(r,e):=\sum_{\ell=1}^K Q_{K,\ell}(e)e^{-\beta\ell r},$ which represents a $K$-term exponential
approximation of the true kernel $\Phi$ and $Q_{K,\ell}(\cdot)\in E\to\mathbb R^{2\times 2}$ denotes the effort effect on the kernel. We set the following assumption for a fixed kernel $\Phi$. 

Let $U:=A\times E\times[0,\overline\nu]$, where
$A$ and $E$ are compact, and assume $\varepsilon\in E$.
The original and approximating systems have the same
initial state and regime, with empty incident history.

\begin{assumption}[Kernel approximation]
\label{ass:ma-kernel-approximation}
    For a fixed $\beta>0$, the selected approximants satisfy
\begin{equation*}
 \begin{aligned}
 \qquad \epsilon_K:=\int_0^T d_K(r)\,dr\underset{K\to\infty}{\longrightarrow}0,\quad
 d_K(r)&:=\sup_{e\in E}\|\Phi^K(r,e)-\Phi(r,e)\|.
 \end{aligned}
\end{equation*}
\end{assumption}
We refer to Remark 3.2 in ~\cite{bielecki2026continuoustimereinforcementlearningcontrolled} for the sufficient condition satisfying the above assumption.
In equations shared by both
models, use $q\in\{0,K\}$, with $\Phi^0:=\Phi$ and the convention
\[
 (\mathbf{S}^0,N^0,\lambda^0,J^0,I^0,\pi^0)=(\mathbf{S},N,\lambda,J,I,\pi).
\]
The same convention applies to components and derived quantities.
Use a common Brownian motion $B=(W^0,W^{AI},W^{R})$, independent marked Poisson random
measures $\Pi^j$ of intensity $dt\,d\theta\,\rho_j(dm)$,
$j\in\{S,A\}$, and an independent switching measure $\Pi^{\mathrm{sw}}$
of intensity $dt\,d\theta$. Here each $\rho_j$ is a probability law.
Write $\Pi=(\Pi^S,\Pi^A,\Pi^{\mathrm{sw}})$ for this common Poisson
embedding; it is shared by the original and approximating systems.
We define
\begin{align}
 &\boldsymbol{\lambda}_t^q
 =\boldsymbol\mu(t,\bar e_t^q)+\int_{(0,t)}
        \boldsymbol\Phi^q(t-s,\bar e_t^q)\,d\boldsymbol N_s^q,
        \label{eq:ma-intensity}\\
 &\notag M^{q,j}(dt,dm)
 =\int_{\mathbb R_+}\mathbf1_{\{\theta\le\lambda_t^{q,j}\}}
          \Pi^j(dt,d\theta,dm),
 \qquad N_t^{q,j}=M^{q,j}((0,t]\times\mathcal M_j),\\
 &\notag dJ_t^q=\int_{\mathbb R_+}\mathbf1_{\{\theta\le\nu_t^q\}}
             \Pi^{\mathrm{sw}}(dt,d\theta),
 \qquad dI_t^q=(1-2I_{t-}^q)\,dJ_t^q,\\
 &d\mathbf{S}_t^q
 \!=\!b_{\mathbf{S}}(t,\mathbf{S}_{t-}^q,I_{t-}^q,\pi_t^q,\boldsymbol\lambda_t^q)dt
   \!+\!\Sigma_{\mathbf{S}}(t,\mathbf{S}_{t-}^q,I_{t-}^q,\pi_t^q,\boldsymbol\lambda_t^q)dB_t\!+\!\sum_j\int_{\mathcal M_j}mM^{q,j}(dt,dm)
        \label{eq:ma-dynamics}
\end{align}
where 
\begin{align*}
b_{\mathbf S}(t,\mathbf s,i,\pi,\ell)&\!:=\!
 \begin{pmatrix}
   (1-i)b(t,x,a)+i b(t,x,\widehat a_t)\\
   -\theta _0R_t
 \end{pmatrix}\\
 \Sigma_{\mathbf S}(t,\mathbf s,i,\pi,\ell)
 &\!:=\!
 \begin{pmatrix}
   (1-i)\sigma(t,x,a)+i\sigma(t,x,\widehat a)
     & (1-i)\gamma(t,x,e,\ell)
+i\gamma(t,x,\varepsilon,\ell)
     & 0\\
   0
     & 0
     & \theta_1 s(L_t)
 \end{pmatrix}.
 \end{align*}
Write
$\mathbf{s}=(x,r)$ for a generic joint state.

\begin{assumption}[Kernel regularity]
\label{ass:ma-kernel}
The baseline $\mu(t,e)$ is deterministic, nonnegative, bounded, and
continuous. The kernels are measurable and entrywise nonnegative on
$[0,T]\times E$. There exist finite constants $H,L>0$, independent
of $K$ and the controller, such that, for all $K\ge1$,
$q\in\{0,K\}$, $t,r\in[0,T]$, and $e,e'\in E$,
\begin{align*}
 \sup_{r,e}\|\Phi^q(r,e)\|&\le H,\\
 \|\mu(t,e)-\mu(t,e')\|_1
   +\|\Phi^q(r,e)-\Phi^q(r,e')\|
 &\le L|e-e'|.
\end{align*}
\end{assumption}

\begin{assumption}[Frozen controller and dynamics]
\label{ass:ma-dynamics}
All coefficient maps are measurable. The map $G$ is uniformly
Lipschitz in $\mathbf{s}$, with $\sup_t|G(t,0)|<\infty$.
For each regime $i$, uniformly in $t$,
\begin{align}
 &|b_{\mathbf{S}}(t,\mathbf{s},i,\pi,\ell)-b_{\mathbf{S}}(t,\mathbf{s}',i,\pi',\ell')|
  +\|\Sigma_{\mathbf{S}}(t,\mathbf{s},i,\pi,\ell)-\Sigma_{\mathbf{S}}(t,\mathbf{s}',i,\pi',\ell')\|_F
 \notag\\
 &\hspace{15mm}\le L\bigl(|\mathbf{s}-\mathbf{s}'|+|\pi-\pi'|+\|\ell-\ell'\|_1\bigr),
       \label{eq:ma-coefficient-lip}\\
 &\notag|b_{\mathbf{S}}(t,0,i,\pi,0)|+\|\Sigma_{\mathbf{S}}(t,0,i,\pi,0)\|_F\le L.
       \label{eq:ma-coefficient-growth}
\end{align}
The jump amplitudes are state- and control-independent, with
\begin{equation}
\label{eq:ma-bound-jump}
    \sum_j\int|m|^4\rho_j(dm)<\infty.
\end{equation}
The open-loop and feedback systems considered below admit unique
nonexplosive strong solutions and satisfy
\begin{equation}
 \sup_{K\ge1}\sup_{q\in\{0,K\},\pi}\mathbb E\left[
    \sup_{t\le T}|\mathbf{S}_t^{q,\pi}|^4+\|N_T^{q,\pi}\|_1^4
       +\int_0^T\|\lambda_t^{q,\pi}\|_1^4\,dt\right]<\infty.
 \label{eq:ma-moments}
\end{equation}
\end{assumption}

% Point 2: unnumbered to preserve Assumption A.3 and Theorem A.6.
\begin{remark*}[Sufficient conditions]
Suppose Assumption~\ref{ass:ma-kernel} and
\eqref{eq:ma-coefficient-lip}--\eqref{eq:ma-bound-jump} hold,
$\mathbb E[|\mathbf S_0|^4]<\infty$, and controls take values in a compact set with $0\le\nu\le\bar\nu<\infty$. For feedback controls, assume that after substituting the policy into
both the coefficients and the intensities, the resulting drift and
diffusion are locally Lipschitz in $(X,R)$ between jumps, for each fixed
regime and past incident history, uniformly in time on bounded sets.
These conditions suffice for unique nonexplosive strong solutions and
\eqref{eq:ma-moments}, uniformly in $K$ and the policy. This is a bounded-kernel variant of the sufficient conditions in
\cite[Proposition~3.6]{bielecki2026continuoustimereinforcementlearningcontrolled},
allowing intensity-dependent diffusion and state- and control-independent
marks with finite fourth moments. No subcriticality condition is needed
on the finite horizon $[0,T]$.
\end{remark*}

\begin{assumption}[Rewards]
\label{ass:ma-reward}

Let $r_0(t,\mathbf{s},i,\pi)=f(t,x)-c(t,\mathbf{s},i,\pi)$, where the operating cost includes
the effective AI action. For every regime $i$,
\begin{align*}
 |r_0(t,\mathbf{s},i,\pi)-r_0(t,\mathbf{s}',i,\pi')|+|g(x)-g(x')|&\le L(1+|\mathbf{s}|+|\mathbf{s}'|)(|\mathbf{s}-\mathbf{s}'|+|\pi-\pi'|),\\
 |r_0(t,0,i,\pi)|+|g(0)|&\le L.
\end{align*}
\end{assumption}

\paragraph{Equivalent expected objective.}
Put $m_{2,j}^X=\int|m_X|^2\rho_j(dm)$ and define
\begin{equation}
 \begin{split}
 \mathfrak r(t,\mathbf{s},i,\pi,\ell)
 ={}&r_0(t,\mathbf{s},i,\pi)-\frac{\kappa(i)}2\nu^2-\chi_{i,1-i}\nu-\frac\eta2\left(
     \|\Sigma_X(t,\mathbf{s},i,\pi,\ell)\|^2+\sum_jm_{2,j}^X\ell_j
                      \right).
 \end{split}
 \label{eq:ma-effective-reward}
\end{equation}
Compensation of the incident and switching measures gives exactly
\begin{equation}
 J^{q,\mathrm{rel}}(\pi)
 =\mathbb E\left[g(X_T^q)
    +\int_0^T\mathfrak r(t,\mathbf{S}_{t-}^q,I_{t-}^q,\pi_t^q,\lambda_t^q)\,dt
              \right].
 \label{eq:ma-expected-objective}
\end{equation}
As a consequence of Assumption \ref{ass:ma-reward} for any fixed $i\in\{0,1\}$ we have
\begin{equation*}
 \begin{split}
 &|\mathfrak r(t,\mathbf{s},i,\pi,\ell)-\mathfrak r(t,\mathbf{s}',i,\pi',\ell')|\\
 &\quad\le C(1+|\mathbf{s}|+|\mathbf{s}'|+\|\ell\|_1+\|\ell'\|_1)
       (|\mathbf{s}-\mathbf{s}'|+|\pi-\pi'|+\|\ell-\ell'\|_1).
 \end{split}
 \label{eq:ma-effective-reward-lip}
\end{equation*}

\subsection{Filter representation and inherited Hawkes estimates}
\label{app:ma-inherited}

For any incident history $N^j$, $j\in\{S,A\}$ we define
\[
 Z_t^{K,\ell,j}
 =\int_{(0,t]}e^{-\beta\ell(t-s)}\,dN_s^j,
 \qquad \ell=1,\ldots,K.
\]
As in \cite[Section~3.2]{bielecki2026continuoustimereinforcementlearningcontrolled},
\[
 dZ_t^{K,\ell,j}=-\beta\ell Z_t^{K,\ell,j}\,dt+dN_t^j,
 \qquad
 \boldsymbol\lambda_t^K=\mu(t,\bar e_t^K)
       +\sum_{\ell=1}^KQ_{K,\ell}(\bar e_t^K)
                        Z_{t-}^{K,\ell}.
\]

For later use, set
$\mathfrak b_K(t):=\int_{(0,t)}d_K(t-s)\,d\|N_s\|_1$.
Since $d_K\le2H$, pathwise
\[
 \int_0^T\mathfrak b_K(t)\,dt\le\epsilon_K\|N_T\|_1,
 \qquad
 \sup_{t\le T}\mathfrak b_K(t)\le2H\|N_T\|_1.
\]
Consequently, uniformly over the controls considered,
\begin{equation}
 \mathbb E\left[\int_0^T
       \bigl(\mathfrak b_K(t)+\mathfrak b_K(t)^2\bigr)\,dt\right]
 \le C\epsilon_K.
 \label{eq:ma-residual}
\end{equation}
In particular, evaluating $\Phi^K$ on the original incident history,
with the original effort, approximates its intensity in
$L^1(\mathbb P\otimes dt)$ with error at most $C\epsilon_K$.

\begin{lemma}[Common-control intensity estimate]
\label{lem:ma-inherited}
If $\pi^K=\pi$ is the same predictable control process, then
\begin{equation}
 A_K:=\sup_\pi\mathbb E\left[\int_0^T
             \|\lambda_t^{K,\pi}-\lambda_t^{\pi}\|_1\,dt\right]
       \le C\epsilon_K.
 \label{eq:ma-intensity-L1}
\end{equation}
\end{lemma}

\begin{proof}
Fix a predictable control $\pi^K=\pi$.
Note that 
$J^K=J$ and $I^K=I$, hence $\bar e^K=\bar e$.
The intensity comparison in
\cite[proof of Theorem~3.10, Step~2]
{bielecki2026continuoustimereinforcementlearningcontrolled}
therefore applies with the current control replaced by $\bar e$ and no
baseline state-error term, since $\mu(t,e)$ is independent of $\mathbf S$.
Together with \eqref{eq:ma-residual} and the uniform resolvent
bound in \cite[Lemma~3.1]
{bielecki2026continuoustimereinforcementlearningcontrolled}, this gives
\eqref{eq:ma-intensity-L1}. The constants are independent of $K$ and
$\pi$, so the estimate is uniform over the admissible controls.
\end{proof}

\subsection{Full-information approximation}
\label{app:ma-full}

Let $\mathbb F^{\mathrm{full}}$ be the usual augmentation
of the filtration generated by the common initial variables
$(\mathbf S_0,I_0)$ and the common driving noises. Let $\mathcal A^{\mathrm{full}}$ contain all
$U$-valued $\mathbb F^{\mathrm{full}}$-predictable controls.
Define
\[
 V^{q,\mathrm{full}}
   :=\sup_{\pi\in\mathcal A^{\mathrm{full}}}J^{q,\mathrm{rel}}(\pi).
\]
\textcolor{black}{These values optimize over the common strong full-information control class; the observable policy class is defined separately below.} The following theorem extends \cite[Theorem 3.10]{bielecki2026continuoustimereinforcementlearningcontrolled} to turbulence risk adding the intensity process into the volatility of the controlled process $X$.

\begin{theorem}[Full-information convergence]
\label{thm:ma-full}
Under Assumptions~\ref{ass:ma-kernel-approximation}--\ref{ass:ma-reward}, as $K\to\infty$,
\begin{align}
 \sup_{\pi\in\mathcal A^{\mathrm{full}}}\mathbb E\left[
    \sup_{t\le T}|\mathbf{S}_t^{K,\pi}-\mathbf{S}_t^{\pi}|^2
       +\int_0^T\|\lambda_t^{K,\pi}-\lambda_t^{\pi}\|_1^2\,dt
        \right]&\longrightarrow0,
       \label{eq:ma-full-process}\\
 \notag\Delta_K:=\sup_{\pi\in\mathcal A^{\mathrm{full}}}
       |J^{K,\mathrm{rel}}(\pi)-J^{\mathrm{rel}}(\pi)|&\longrightarrow0.
\end{align}
Consequently $|V^{K,\mathrm{full}}-V^{\mathrm{full}}|\le\Delta_K\to0$.
\end{theorem}

\begin{proof}
All suprema over $\pi$ below are over $\mathcal A^{\mathrm{full}}$.
Fix a predictable control $\pi$. In this proof, we use the notation $C>0$ for a generic constant which can varies lines to lines and independent of the choice of $\pi$ and $K$. The switching construction gives
$I^K=I$. For the sake of simplicity we remove the superscript $\pi$ in the notation and write
$\Delta\mathbf S=\mathbf S^K-\mathbf S$ and
$\Delta\lambda=\lambda^K-\lambda$.
By Lemma~\ref{lem:ma-inherited}, $A_K\le C\epsilon_K\to0$.
Equation \eqref{eq:ma-moments} gives
$\sup_{K,\pi}\mathbb E[\int_0^T\|\Delta\lambda_t\|_1^4\,dt]<\infty$.
H\"older's inequality therefore yields
\begin{equation}
 B_K:=\sup_\pi\mathbb E\left[\int_0^T\|\Delta\lambda_t\|_1^2\,dt\right]
 \le C A_K^{2/3}\le C\epsilon_K^{2/3}\longrightarrow0.
 \label{eq:ma-intensity-L2}
\end{equation}
We define for $j\in \{S,A\}$
\begin{align*}
 \Delta b_s
 &:=b_{\mathbf S}(s,\mathbf S_{s-}^K,I_{s-},\pi_s,\lambda_s^K)
   -b_{\mathbf S}(s,\mathbf S_{s-},I_{s-},\pi_s,\lambda_s),\\
 \Delta\Sigma_s
 &:=\Sigma_{\mathbf S}(s,\mathbf S_{s-}^K,I_{s-},\pi_s,\lambda_s^K)
   -\Sigma_{\mathbf S}(s,\mathbf S_{s-},I_{s-},\pi_s,\lambda_s),\\
 \delta_j^K(s,\theta)
 &:=\mathbf1_{\{\theta\le\lambda_s^{K,j}\}}
     -\mathbf1_{\{\theta\le\lambda_s^j\}},\qquad
 \widetilde\Pi^j(ds,d\theta,dm)
 :=\Pi^j(ds,d\theta,dm)-ds\,d\theta\,\rho_j(dm).
\end{align*}
From \eqref{eq:ma-dynamics} we get
\begin{equation*}
 \Delta\mathbf S_t
 =\mathcal D_t^K+\mathcal B_t^K+\mathcal M_t^K+\mathcal C_t^K, \label{eq:Delta-S}
\end{equation*}
where
\begin{align*}
 \mathcal D_t^K&:=\int_0^t\Delta b_s\,ds,
 \,\mathcal B_t^K:=\int_0^t\Delta\Sigma_s\,dB_s,\,
 \mathcal M_t^K:=\sum_j\int_0^t\int_{\mathbb R_+\times\mathcal M_j}
       m\delta_j^K(s,\theta)\,\widetilde\Pi^j(ds,d\theta,dm),\\
 \mathcal C_t^K&:=\sum_j\bar m_j\int_0^t\Delta\lambda_s^j\,ds,
 \quad\bar m_j:=\int_{\mathcal M_j}m\rho_j(dm).
\end{align*}
Set
$F_K^\pi(t):=\mathbb E[\sup_{r\le t}|\Delta\mathbf S_r|^2]$, hence 

\[ F_K^\pi(t)\leq 4(\mathbb E[\sup_{r\le t}|\mathcal D^K_r|^2]+\mathbb E[\sup_{r\le t}|\mathcal B^K_r|^2]+\mathbb E[\sup_{r\le t}|\mathcal M^K_r|^2]+\mathbb E[\sup_{r\le t}|\mathcal C^K_r|^2]). \]
Note that $b_{\mathbf S}$ does not depend on the intensity.
Similarly
of \cite[proof of Lemma~3.8, Equation (17)]
{bielecki2026continuoustimereinforcementlearningcontrolled}
we get $|\Delta b_s|\le L|\Delta\mathbf S_{s-}|$.
From Cauchy-Schwarz Inequality we have
\begin{equation*}
 \begin{aligned}
 \mathbb E\left[\sup_{r\le t}|\mathcal D_r^K|^2\right]
 &\le t\,\mathbb E\left[\int_0^t|\Delta b_s|^2\,ds\right]\\
 &\le TL^2\int_0^tF_K^\pi(s)\,ds.
 \end{aligned}
 \label{eq:ma-drift-estimate}
\end{equation*}

Then, by adapting the martingale estimate in
\cite[proof of Lemma~3.8, equations~(18)--(19)]
{bielecki2026continuoustimereinforcementlearningcontrolled}
in $L^2$ together with Doob's inequality, then It\^o's isometry,
we get
\begin{equation*}
 \begin{aligned}
 \mathbb E\left[\sup_{r\le t}|\mathcal B_r^K|^2\right]
 &\le4\,\mathbb E\left[\int_0^t\|\Delta\Sigma_s\|_F^2\,ds\right]\\
 &\le8L^2\int_0^tF_K^\pi(s)\,ds+8L^2 B_K.
 \end{aligned}
 \label{eq:ma-brownian-estimate}
\end{equation*}
Note that the additional $B_K$ term accounts for the intensity-dependent diffusion specific to our problem compared with \cite{bielecki2026continuoustimereinforcementlearningcontrolled}.

The thinning identity underlying
\cite[proof of Lemma~3.8, equation~(14)]
{bielecki2026continuoustimereinforcementlearningcontrolled} gives
\[
 \int_{\mathbb R_+}|\delta_j^K(s,\theta)|^2\,d\theta
 =|\Delta\lambda_s^j|,
\]
since $|\delta_j^K|^2=|\delta_j^K|$.
Doob's inequality and the marked-Poisson isometry together with \eqref{eq:ma-bound-jump} therefore give
\begin{equation*}
 \begin{aligned}
 \mathbb E\left[\sup_{r\le t}|\mathcal M_r^K|^2\right]
 &\le4\sum_j\left(\int_{\mathcal M_j}|m|^2\rho_j(dm)\right)
          \mathbb E\left[\int_0^t|\Delta\lambda_s^j|\,ds\right]\\
 &\le C A_K.
 \end{aligned}
 \label{eq:M-estimate}
\end{equation*}

Then, using $\int_{\mathbb R_+}\delta_j^K(s,\theta)\,d\theta
=\Delta\lambda_s^j$ and
Cauchy-Schwarz Inequality, we get
\begin{equation*}
 \begin{aligned}
 \mathbb E\left[\sup_{r\le t}|\mathcal C_r^K|^2\right]
 &\le t\left(\sum_j|\bar m_j|^2\right)
       \mathbb E\left[\int_0^t\|\Delta\lambda_s\|_1^2\,ds\right]\\
 &\le C B_K.
 \end{aligned}
 \label{eq:C-estimate}
\end{equation*}

Combining the four estimates with
$|v_1+v_2+v_3+v_4|^2\le4\sum_{k=1}^4|v_k|^2$ gives
\[
 F_K^\pi(t)\le C\int_0^tF_K^\pi(s)\,ds+C(A_K+B_K).
\]
 Gronwall's inequality
therefore implies
\[
 \sup_\pi F_K^\pi(T)\le C(A_K+B_K)\longrightarrow0.
\]
Together with \eqref{eq:ma-intensity-L2}, this proves
\eqref{eq:ma-full-process}.
Finally, adapting  
\cite[proof of Theorem~3.10, Step~5]
{bielecki2026continuoustimereinforcementlearningcontrolled}, with the
effective reward \eqref{eq:ma-effective-reward}, Assumption~\ref{ass:ma-reward},
\eqref{eq:ma-moments}, and Cauchy-Schwarz' Inequality, we get
\[
 |J^{K,\mathrm{rel}}(\pi)-J^{\mathrm{rel}}(\pi)|
 \le C\bigl(F_K^\pi(T)+B_K\bigr)^{1/2}.
\]
Thus $\Delta_K\le C(\sup_\pi F_K^\pi(T)+B_K)^{1/2}\to0$, therefore
$|V^{K,\mathrm{full}}-V^{\mathrm{full}}|\le\Delta_K\to0$.
\end{proof}

% Paste immediately after the proof of Theorem \ref{thm:ma-full},
% before \section{Numerical Implementation and Benchmarks}.
% Uses the theorem environments and packages already in the supplied draft.
% Append the entries in appendix_A4_references.bib to ref.bib.
% Two corrections to the existing Appendix A are needed when integrating:
% (1) In A.1 use m_{2,j}^X = \int |m_X|^2\rho_j(dm), not \int |m|^2\rho_j(dm).
% (2) Replace the introduction's claim of observable-value convergence by:
% Finally, Section~\ref{app:ma-observable} bounds the value loss when controls
% depend only on the current observable feature vector.

\subsection{Information loss under current-observation feedback}
\label{app:ma-observable}

Fix $K\ge1$ and a deterministic frozen recommendation (black-box feedback from the AI decision) $G$ defined by  $G(t,\mathbf S_t):=\widehat a_{t}^{\rm raw}=a^*(t,X_t)+R_t$, and impose
Assumptions~\ref{ass:ma-kernel-approximation}--\ref{ass:ma-reward}.
We use the same notations that the one introduced in
Sections~\ref{app:ma-model}--\ref{app:ma-inherited}.
Let $O_t^K$ be the current observation vector, \textit{i.e.} $O_t=(t,X_t,I_t,Z_t,N_t^S,N_t^A,\mathbf1_{\{I_t=1\}}\widehat a^{\rm raw})$.
The raw recommendation
$I_t^K G(t,\mathbf S_t^K)$ is observed, but not $R_t^K$ separately.
Assume this raw output has a c\`adl\`ag version whose left limit is
$I_{t-}^K G(t,\mathbf S_{t-}^K)$,
$dt\otimes\mathbb P^\pi$-almost everywhere.

Let $\mathcal A^{K,\mathrm{obs}}$ consist of deterministic Borel
feedbacks
\[
 \pi_t=f(t,O_{t-}^K)\in U
\]
whose closed-loop systems are admissible. Define
\[
 V^{K,\mathrm{obs}}
 :=\sup_{\pi\in\mathcal A^{K,\mathrm{obs}}}J^{K,\mathrm{rel}}(\pi).
\]
Note that $ \mathcal A^{K,\mathrm{obs}}\subseteq\mathcal A^{\mathrm{full}}$. In particular, both classes use \eqref{eq:ma-expected-objective}, the same initial law,
and the same continuous-time decision convention but $\mathcal A^{K,\mathrm{obs}}$ does not obsevred $R$ directly.

Write $\lambda=\lambda_i^K(t,z,\pi)$ for the intensity in
Section~\ref{app:ma-inherited}, evaluated at $(i,z,\pi)$.
For a smooth pair $\varphi=(\varphi_0,\varphi_1)$, define the
generator by
\begin{equation*}
 \begin{split}
 (\mathcal L^{K,\pi}\varphi)_i(t,\mathbf s,z)
 :={}&b_{\mathbf S}\cdot D_{\mathbf s}\varphi_i
 +\frac12\operatorname{tr}\!\left(
   \Sigma_{\mathbf S}\Sigma_{\mathbf S}^{\!\top}
                        D_{\mathbf s}^2\varphi_i\right)
 -\sum_{\ell=1}^K\sum_{j\in\{S,A\}}
                 \beta\ell z^{\ell,j}\partial_{z^{\ell,j}}\varphi_i\\
 &+\sum_{j\in\{S,A\}}\lambda^j\int_{\mathcal M_j}
   [\varphi_i(t,\mathbf s+m,z+\delta_j)
                  -\varphi_i(t,\mathbf s,z)]\,\rho_j(dm)\\
 &+\nu[\varphi_{1-i}(t,\mathbf s,z)-\varphi_i(t,\mathbf s,z)]+\mathfrak r(t,\mathbf s,i,\pi,\lambda_i^K(t,z,\pi)),
 \end{split}
 \label{eq:ig-generator}
\end{equation*}
with $\mathfrak r$ from \eqref{eq:ma-effective-reward}. 

\begin{assumption}\label{ass:estimater}
    We assume that there exists an estimator $h$ for which is a deterministic Borel-measurable
map $h(t,o)\in\mathbb R$, with associated estimate
\[
 \widehat R_t^{K,h}:=h(t,O_{t-}^K).,\] 
such that for any policy $\pi\in \mathcal A^{K,obs}$
\[
 \mathbb E^{\pi}\int_0^T
       |R_{t-}^K-\widehat R_t^{K,h}|^2\,dt<\infty.
\]
\end{assumption}

\begin{assumption}[Verification, sensitivity, and feedback selection]
\label{ass:ig-verification}
For the fixed $K$, the following hold.
\begin{enumerate}
\item[(i)] There exists a unique pair of solutions $(v^K_0,v_1^K)$ which are continuously differentiable with respect to $(t,z)$ and twice continuously differentiable with respect to
$(x,r)$, to the coupled integro-partial HJB equation for $i\in \{0,1\}$
\begin{equation*}
 \partial_t v_i^K+\sup_{u\in U}\mathcal L^{K,u} v^K_i(\cdot;\pi)=0,
 \qquad v_i^K(T,\mathbf s,z)=g(x),
 \label{eq:ig-HJB}
\end{equation*}
such that
\[
 V^{K,\mathrm{full}}
 =v_{I_0}^K(0,\mathbf S_0,\mathcal Z_0^K).
\] Moreover, $v_i^K$ is continuous up to the terminal time, and there
exists a finite constant $C_K$ such that
\[
|v_i^K(t,s,z)|
\le C_K\bigl(1+|s|^2+\|z\|^2\bigr),
\qquad
(t,s,z)\in[0,T]\times\mathbb R^2\times\mathcal Z_K,
\quad i\in\{0,1\}.
\]

\item[(ii)] The Hamiltonian is Borel measurable and continuous in the
control; its control suprema are measurable. For some $L_K<\infty$,
\begin{equation*}
 |\mathcal L^{K,\pi} v^K_i(t,x,z,r;u)
       -\mathcal L^{K,\pi} v^K_i(t,x,z,r';u)|
 \le L_K|r-r'|,
 \qquad u\in U.
 \label{eq:ig-sensitivity}
\end{equation*}
\item[(iii)] There exists a deterministic Borel feedback $\hat\pi^{\rm obs}_t$ such
that $\hat\pi^{\rm obs}_t=\hat\pi^{*,\rm obs}(t,O_{t-}^K)$ belongs to
$\mathcal A^{K,\mathrm{obs}}$ and satisfies
\begin{equation*}
 \hat\pi^{*,\rm obs}(t,o)\in\arg\max_{u\in U}
 \mathcal L^{K,\pi} v^K_i
       (t,x,z,h(t,o),u).
 \label{eq:ig-selector}
\end{equation*}

\end{enumerate}
\end{assumption}
\begin{remark}
    Assumption \ref{ass:estimater} is structural to our problem and depends on the capacity of the supervisor to monitor AI error estimation. Assumption \ref{ass:ig-verification} is a classical-solution (verification) hypothesis: when a classical solution exists, the argument of \cite{bielecki2026continuoustimereinforcementlearningcontrolled} identifies it with $V^{K,full}$. \textcolor{black}{The result assumes existence of the stated classical solution for the coupled system.}\end{remark}

We now set
\begin{equation*}
 B_h^K(t):=\mathbb E^{\hat\pi^{*,\rm obs}}|R_{t-}^K-\widehat R_t^{K,h}|^2.
 \label{eq:ig-reconstruction-error}
\end{equation*}
\begin{remark}
Conditioning on the current vector $O_{t-}^K$, rather than its history,
gives the orthogonal decomposition
\begin{equation*}
 \begin{split}
 B_h^K(t)
 ={}&\mathbb E^{\hat\pi^{*,\rm obs}}\operatorname{Var}^{\hat\pi^{*,\rm obs}}
                        (R_{t-}^K\mid O_{t-}^K)+\mathbb E^{\hat\pi^{*,\rm obs}}\left|
       \mathbb E^{\hat\pi^{*,\rm obs}}[R_{t-}^K\mid O_{t-}^K]
                          -h(t,O_{t-}^K)\right|^2.
 \end{split}
 \label{eq:ig-MSE-variance}
\end{equation*}
Thus $B_h^K(t)$ equals the expected conditional variance precisely when
$h(t,O_{t-}^K)$ is the conditional mean under the comparison policy.
\end{remark}
\begin{theorem}[Current-feedback value bound for fixed $G$]
\label{thm:ig-variance}
Under the assumptions above, there exists a function $L_K>0$ such that
\begin{equation*}
 \begin{split}
 0\le V^{K,\mathrm{full}}-V^{K,\mathrm{obs}}
 &\le 2L_K\int_0^T\sqrt{B_h^K(t)}\,dt.
 \end{split}
 \label{eq:ig-value-bound}
\end{equation*}
\end{theorem}

\begin{proof}
We set
\[
\widehat{\mathbf S}_t^{K,h}:=(X_{t-}^K,\widehat R_t^{K,h}).
\]

We define the Hamiltonian
\[\mathcal H_i^{K}(t,\mathbf s,z):= \sup_{u\in U} \bigl(\mathcal L^{K,u}v^K\bigr)_{i}
    (t,\mathbf s,z),\]
    and
\[
\mathcal R_t^{K,h}
:=
\mathcal H^{K}_{I_{t-}^K}
    (t,\mathbf S_{t-}^K,Z_{t-}^K)
-
\bigl(\mathcal L^{K,\pi_t^{*,\mathrm{obs}}}v^K\bigr)_{I_{t-}^K}
    (t,\mathbf S_{t-}^K,Z_{t-}^K).
\]
Using the maximizing property at $\widehat{\mathbf S}_t^{K,h}$ gives
\begin{align*}
\mathcal R_t^{K,h}
={}&
\mathcal H^{K}_{I_{t-}^K}
    (t,\mathbf S_{t-}^K,Z_{t-}^K)
-
\mathcal H^{K}_{I_{t-}^K}
    (t,\widehat{\mathbf S}_{t-}^K,Z_{t-}^K)
\\
&+
\bigl(\mathcal L^{K,\pi_t^{*,\mathrm{obs}}}v^K\bigr)_{I_{t-}^K}
    (t,\widehat{\mathbf S}_t^{K,h},Z_{t-}^K)
-
\bigl(\mathcal L^{K,\pi_t^{*,\mathrm{obs}}}v^K\bigr)_{I_{t-}^K}
    (t,\mathbf S_{t-}^K,Z_{t-}^K).
\end{align*}
Recalling that $
\left|\sup_{u\in U}a_u-\sup_{u\in U}b_u\right|
\le \sup_{u\in U}|a_u-b_u|,$ together with
Assumption \ref{ass:ig-verification}(ii) implies
\begin{equation*}
0\le \mathcal R_t^{K,h}
\le 2L_K\left|R_{t-}^K-\widehat R_t^{K,h}\right|.
\label{eq:reconstruction-residual-bound}
\end{equation*}
In particular, Assumption \ref{ass:estimater} yields
\[
\mathbb E^{\pi^{*,\mathrm{obs}}}\Big[]
\int_0^T \mathcal R_t^{K,h}\,dt\Big]<\infty.
\]

For the sake of simplicity, we denote the effective running reward by
\[
r_t
:=
r\bigl(t,S_{t-}^K,I_{t-}^K,
       \pi_t^{*,\mathrm{obs}},\lambda_t^K\bigr).
\]
Applying the jump-diffusion It\^o formula to
$v_{I_t^K}^K(t,\mathbf S_t^K,Z_t^K)$, 
we obtain from the integro-partial HJB equation
\[
v_{I_t^K}^K(t,\mathbf S_t^K,Z_t^K)
+\int_0^t r_s\,ds
=
v_{I_0}^K(0,\mathbf S_0,Z_0^K)
-\int_0^t\mathcal R_s^{K,h}\,ds
+M_t,
\]
where $M$ is a local martingale starting at zero.

From Assumption \ref{ass:ig-verification}(i), Assumption \ref{ass:ma-dynamics}, and $Z_t^{K,\ell,j}\le N_t^{K,j}$ we get
\[
\mathbb E^{\pi^{*,\mathrm{obs}}}
[\sup_{t\le T}
|v_{I_t^K}^K(t,\mathbf S_t^K,Z_t^K)|] + \mathbb E^{\pi^{*,\mathrm{obs}}}[
\int_0^T |r_t|\,dt]<\infty.
\]

Using a localization approach, the martingale term has
zero expectation. Using the terminal condition and the verification
identity in Assumption \ref{ass:ig-verification}(i), we conclude that
\begin{equation*}
V^{K,\mathrm{full}}
-
J^{K,\mathrm{rel}}(\pi^{*,\mathrm{obs}})
=
\mathbb E^{\pi^{*,\mathrm{obs}}}
[\int_0^T\mathcal R_t^{K,h}\,dt].
\label{eq:current-feedback-performance-gap}
\end{equation*}

Finally, since
$\pi^{*,\mathrm{obs}}\in\mathcal A^{K,\mathrm{obs}}
\subseteq\mathcal A^{\mathrm{full}}$,
\begin{align*}
0
&\le V^{K,\mathrm{full}}-V^{K,\mathrm{obs}}
\\
&\le V^{K,\mathrm{full}}
     -J^{K,\mathrm{rel}}(\pi^{*,\mathrm{obs}})
\\
&\le 2L_K\int_0^T
\mathbb E^{\pi^{*,\mathrm{obs}}}
\left|R_{t-}^K-\widehat R_t^{K,h}\right|\,dt
\\
&\le 2L_K\int_0^T\sqrt{B_h^K(t)}\,dt,
\end{align*}
where the last inequality follows from Cauchy--Schwarz.
The right-hand side is finite because Assumption \ref{ass:estimater} gives
\[
\int_0^T\sqrt{B_h^K(t)}\,dt
\le
\left(
T\,\mathbb E^{\pi^{*,\mathrm{obs}}}
\int_0^T
\left|R_{t-}^K-\widehat R_t^{K,h}\right|^2\,dt
\right)^{1/2}
<\infty.
\]
\end{proof}

\begin{remark}
    The upper bound on the error made between the value function with full observation, including AI error observation and the value with estimated AI error depends on the variance accuracy of this estimator. Variance reductiton technics can be applied but are out of the scope of this study and we let this point for future research. 
\end{remark}

\subsection{Error decomposition for the learned policy}
\label{app:learned-policy-error}

Fix $K \geq 1$ and the initial law used in the preceding sections.
Let $\widehat{\pi}_K \in \mathcal{A}^{K,\mathrm{obs}}$ be an
admissible learned current-observation feedback, evaluated in
the $K$-term approximating model under the same continuous-time
decision convention. Define its remaining suboptimality by
\begin{equation*}
    \varepsilon_{\mathrm{learn}}^K
    :=
    V^{K,\mathrm{obs}}
    -
    J^{K,\mathrm{rel}}(\widehat{\pi}_K)
    \geq 0.
    \label{eq:learning-residual}
\end{equation*}
The discrepancy from the original full-information optimum
admits the exact decomposition
\begin{equation*}
\begin{aligned}
    V^{\mathrm{full}}
    - J^{K,\mathrm{rel}}(\widehat{\pi}_K)
    ={}&
    \bigl(V^{\mathrm{full}}-V^{K,\mathrm{full}}\bigr)
    +
    \bigl(V^{K,\mathrm{full}}-V^{K,\mathrm{obs}}\bigr)
    \\
    &\quad + \varepsilon_{\mathrm{learn}}^K.
\end{aligned}
\label{eq:total-error-decomposition}
\end{equation*}
The first difference need not be nonnegative, but its absolute
value is bounded by $\Delta_K$ from Theorem~\ref{thm:ma-full}.
Consequently, under the assumptions of Theorem~\ref{thm:ig-variance},
\begin{equation}
\begin{aligned}
    V^{\mathrm{full}}
    - J^{K,\mathrm{rel}}(\widehat{\pi}_K)
    \leq{}&
    \underbrace{\Delta_K}_{\text{kernel approximation}}
    +
    \underbrace{
        2L_K \int_0^T \sqrt{B_h^K(t)}\,dt
    }_{\text{current-observation restriction}} +
    \underbrace{
        \varepsilon_{\mathrm{learn}}^K
    }_{\text{neural approximation and training}}.
\end{aligned}
\label{eq:total-error-bound}
\end{equation}

This decomposition separates the errors addressed by the
theoretical analysis from those associated with the numerical
policy solver. Section~\ref{app:ma-full} controls the kernel-approximation
term, with $\Delta_K \to 0$, while Section~\ref{app:ma-observable} bounds the loss
from restricting control to the current observation features.
\textcolor{black}{The observation-restriction term can persist as $K$ increases.}

The residual $\varepsilon_{\mathrm{learn}}^K$ includes both
the approximation error of the neural feedback class and
the suboptimality remaining after finite training.
\textcolor{black}{The bound retains this residual explicitly, separating the trained policy's neural approximation and finite-training error from the two analytically controlled terms.}

In \eqref{eq:total-error-bound}, $\widehat{\pi}_K$ denotes
the continuously evaluated feedback associated with the
learned network. Its sampled deployment and numerical
simulation, as well as evaluation under the original kernel
when applicable, introduce additional discrepancies that
are separate from the three terms displayed above.

\subsection{Markovianization procedure and Exponential-mixture fit} The first step of the Hawkes-PPO algorithm requires to approach the non-Markovian kernel $\Phi$ with a mixture of exponential. In this context, since the value of the problem with non-Markovian kernel is approached with a mixture of Markovian exponential kernels, we refer to Algorithm 1 in \cite{bielecki2026continuoustimereinforcementlearningcontrolled} for this fitting preliminary procedure, mathematically justified by the results of this Appendix. We use the same procedure for updating the memory process $Z$ in Algorithm~\ref{alg:hawkes-ppo}.

\section{Numerical Implementation and Benchmarks}
\label{app:numerical-details}
This appendix specifies the common numerical environment, the construction of the
AI recommender, the observation-based RL methods, and the full-information
benchmarks. All reported outcomes use the same physical model of the AI black box. The full-information value functions below define the oracle
\textcolor{black}{benchmarks for comparison with the observation-based controller in the main text.}

\subsection{Choice of parameters and evaluation protocol}\label{app:choice}

\begin{table}[H]
\centering
\caption{Physical, AI and simulation parameters.
Kernel rows/columns follow $(S,A)$.
Learning settings are in the appendix.}
\label{tab:num-parameters}
\begingroup
\fontsize{8}{9.5}\selectfont
\setlength{\tabcolsep}{3pt}
\renewcommand{\arraystretch}{1.08}
\begin{tabular}{@{}ll@{\hspace{10pt}}ll@{}}
\toprule
Parameter & Numerical choice
& Parameter & Numerical choice\\
\midrule
$(T,\Delta)$ & $(8,0.05)$
& $(X_0,R_0,Z_0)$ & $(0,0,0)$\\
$(a,e,\nu)$ & $[0,1]^2\times[0,4]$
& $\varepsilon$ & $0.01$\\
$\delta$ & $0.18$
& $(\sigma_0,\sigma_1)$ & $(0.16,0.04)$\\
$f(x)$ & $0.1x+0.05\tanh(x-1)$
& $g(x)$ & $1.2x+0.6\tanh(x-1)$\\
$(k_0,k_a,k_e)$ & $(3.25,0.75,1.25)$
& $k_{AI}$ & $1.5$\\
$(\eta,\kappa,\chi)$ & $(50,0.30,0.20)$
& $\mu$ & $(0.35,0.25)$\\
$M^S$ & ${\rm Unif}[0.10,0.45]$
& $M^A$ & ${\rm Unif}[0.35,0.85]$\\
$(\gamma_0,\gamma_1)$ & $(0.1,1.1)$
& $s(L)$ & $0.0125+\frac{0.0775L}{0.75+L}$\\
$(\theta_0,\theta_1)$ & $(2,2)$ & $A$& $\begin{pmatrix}0.4&0.1\\0.1&0.4\end{pmatrix}$\\
$(\rho_{ij},\beta_{ij})$ & $(2,3/2)$ & Filter rates&
$\beta_k=k/4,\quad k=1,...,20$\\
\bottomrule
\end{tabular}
\endgroup
\end{table}

Controls $(a,e,\nu)\in[0,1]^2\times[0,4]$ are selected every $\Delta=0.05$;
each decision interval comprises five microsteps of $\delta t=0.01$.
Coefficients are frozen within each microstep. Linear state, error, and filter
decay use exponential integration, including event-age decay for jump effects.
We simulate risk events and regime switches by Poisson thinning, using shared candidate streams across policies to reduce noise in comparisons.
The thinning Poisson rate is 32 per channel and the switching rate is 4;
a capacity of 10 events per channel per microstep is checked, with a run
invalidated if its bound or capacity is exceeded.
Switches take effect at microstep endpoints; a fresh supervisory decision occurs
at the next decision epoch. The committed Human action and mitigation can therefore
become active after an AI-to-Human switch within the interval. The current
recommendation $\widehat a_t$ is recomputed at every microstep. All risk events remain observable in both regimes. The public RL observation is
\begin{equation}
 O_t=(t,X_t,I_t,Z_t,N_t^S,N_t^A,\mathbf1_{\{I_t=1\}}\widehat a^{\rm raw}).
 \label{app:observation}
\end{equation}
The full-information oracle additionally observes $R_t$ in both regimes and
knows all model coefficients and the nominal recommendation map.

All displayed objectives are undiscounted.
The realized quadratic variation is accumulated as
\[
 [X]_T=\int_0^T\{\sigma(\alpha_t)^2+\gamma(L_t)^2\}\,dt
       +\sum_{\tau_n^S\leq T}(M_n^S)^2.
\]

% Source: nominal_ai.tex
\subsection{Project-Only AI Recommender: black-box design}
\label{app:nominal-ai}
The nominal recommendation is computed before supervisory training and is
then frozen for all methods. The AI solves a project-only problem on
\((t,x)\), with control \(b\in[0,1]\), dynamics
\[
 dY_s=(b_s-0.18Y_s)\,ds+\sigma(b_s)\,dW_s,\qquad
 \sigma(b)=0.16+0.04b,
\]
and value
\begin{equation*}
 v(t,x)=\sup_b\E_{t,x}\!\left[
 g(Y_T)+\int_t^T
 \left\{f(Y_s)-1.5-b_s^2-\frac{\eta}{2}\sigma(b_s)^2\right\}ds
 \right],\qquad \eta=50. \label{app:nominal-objective}
\end{equation*}
Here \(T=8\), and $f,g$ are the reward same as in Table~\ref{tab:num-parameters}.
The quadratic variations term in this nominal objective is intrinsic project variation:
\([Y]_T-[Y]_t=\int_t^T\sigma(b_s)^2ds\).
The AI ignores Hawkes risk, added risk diffusion, its own
recommendation errors, mitigation, or switching in this optimization.

The project-only HJB is
\begin{equation}
 \begin{split}
 0={}&v_t+\sup_{0\leq b\leq1}
 \left\{(b-0.18x)v_x+\tfrac12\sigma(b)^2v_{xx}
       +f(x)-1.5-b^2-25\sigma(b)^2\right\},\\
 &v(T,x)=g(x). 
 \end{split} \label{app:nominal-hjb}
\end{equation}
We use a DGM network, introduced in \cite{dgm}, approximates \(v\), using automatic
differentiation of the continuous HJB residual and a value parameterization
that imposes the terminal condition exactly. 

For given network derivatives \(p=v_x\) and \(q=v_{xx}\), the
control-dependent Hamiltonian is
\[
 H(b;p,q)=bp-b^2+\left(\frac q2-25\right)(0.16+0.04b)^2.
\]
Its bounded maximum is evaluated by comparing the endpoints \(b=0,1\)
and, when \(2.08-0.0016q>0\), the candidate
\begin{equation*}
 b_{\mathrm{int}}=
 \rm{clip}\!\left(\frac{p+0.0064q-0.32}{2.08-0.0016q},0,1\right).
 \label{app:nominal-greedy}
\end{equation*}

Outside the exported state grid, we approximate the reward by its affine asymptote and use the resulting analytical optimal action, retaining the project quadratic-variation penalty. The extrapolation action is the clipped affine-tail optimum
\[
 \rm{clip}\!\left(
 \frac{0.1/0.18+(1.2-0.1/0.18)e^{-0.18(8-t)}-0.32}{2.08},0,1
 \right).
\]
For a nonconcave Hamiltonian, only the endpoints are required.
The maximizing action defines \(a^*(t,x)\).

\begin{table}[!htbp]
\centering\small
\caption{Project-only recommender training and deployment.}
\label{app:nominal-settings}
\begin{tabular}{@{}p{0.30\linewidth}p{0.65\linewidth}@{}}
\toprule
Setting & Choice\\ \midrule
Network and optimizer & Width 64, one gated DGM block, Adam, gradient norm cap 10.\\
Training & 6,000 updates; batch 2,048; learning rates
\(10^{-3}\) for updates 1--3,000, \(3\cdot10^{-4}\) for 3,001--4,800,
and \(10^{-4}\) for 4,801--6,000.\\
Collocation and selection & \(t\) uniform on \([0,8]\);
\(x\) sampled \(65\%\) from \([-4,6]\) and \(35\%\) from \([-16,16]\).
Residual validation on 8,192 points every 250 updates;
selected update 6,000.\\
Exported policy & Grid \([0,8]\times[-12,12]\), spacings
\(\Delta t=\Delta x=0.01\), bilinear interpolation;
analytic affine-tail action outside the \(x\)-range.\\
Independent check & HJB residual RMSE \(4.596\cdot10^{-4}\);
off-grid action interpolation RMSE \(6.176\cdot10^{-7}\).
Initial recommendation \(a^*(0,0)=0.30080232\).\\
\bottomrule
\end{tabular}
\end{table}

The deployed black box combines this frozen nominal controller with the
persistent error process:
\begin{equation*}
 \begin{split}
 \widehat a_{t} &= \operatorname{clip}_{[0,1]}(\widehat a_{t}^{\rm raw}), \quad\widehat a_{t}^{\rm raw}=
 a^*(t,X_{t})+R_t\\
 dR_t&=-2R_t\,dt+2s(L_t)\,dW_t^R
       -M^AdN^A_t,\\
 s(L)&=0.0125+\frac{0.0775L}{0.75+L},
 \qquad M^A\sim U[0.35,0.85],\qquad R_0=0.
 \end{split} \label{app:recommendation-error}
\end{equation*} 
RL observes $\widehat a_t^{\mathrm{raw}}$ only in AI mode,
while the project implements the clipped action $\widehat a_t$.
Neither $a^\star(t,X_t)$ nor $R_t$ is supplied separately.
The pure autonomous-AI benchmark fixes \(I\equiv1\) and \(\nu\equiv0\) and
uses this same corrupted recommendation, AI mitigation \(0.01\), and the
full supervisory objective. Thus it is evaluated with all Hawkes risks,
the full project QV penalty, and AI cost \(1.5+\alpha_t^2\), despite the
nominal recommender solving only \eqref{app:nominal-hjb}.

% Source: ppo.tex
\subsection{Hawkes-PPO algorithm and foundations}
\label{app:hawkes-ppo}
Hawkes-PPO is PPO \citep{ppooriginal} applied to the public observation
specified in \eqref{app:observation}: time, project state, current regime, event
counts, the common Hawkes filter bank, and the current raw recommendation
only in AI mode. In Human mode the recommendation input is masked. Neither
the nominal controller $a^*$ nor the error state $R$ is supplied to the actor or critic.

There are two independent actor networks and two independent value
networks, indexed by the current regime. They share no trainable
parameters across regimes. A single trained policy is used for both initial
regimes; the network selected along a trajectory changes when the regime
changes. Each network has two width-128 hidden layers with SiLU activations.
For observation \(o\) in regime \(i\), the actor produces three means
\(\mu_{\theta_i}(o)\) and three log standard deviations
\(\ell_{\theta_i}(o)\), clipped to \([-5,1]\), and samples
\begin{equation}
 y\sim\mathcal N\!\left(\mu_{\theta_i}(o),
             \rm{diag}(e^{2\ell_{\theta_i}(o)})\right),\quad(a,e,\nu)=\left(\frac{1+\tanh y_1}{2},
                 \frac{1+\tanh y_2}{2},
                 2(1+\tanh y_3)\right).
\label{app:ppo-policy}
\end{equation}
Let \(r_k\) be the complete interval reward: running benefit less operating
cost, \(\kappa\nu^2/2\) with \(\kappa=0.30\), realized switch fees, and
\(\eta/2=25\) times the QV increment. Terminal benefit is included once,
on termination. Training uses \(\widetilde r_k=r_k/50\).
For a rollout, old value predictions are kept fixed when forming the
generalized advantage estimates:
\begin{align*}
 \delta_k&=\widetilde r_k+(1-d_k)
       V_{\phi^{\mathrm{old}}_{I_{k+1}}}(o_{k+1})
       -V_{\phi^{\mathrm{old}}_{I_k}}(o_k), \\
 A_k&=\delta_k+0.95(1-d_k)A_{k+1},\qquad
 R_k=A_k+V_{\phi^{\mathrm{old}}_{I_k}}(o_k). 
\end{align*}
Here \(d_k\) is the terminal indicator, the discount factor is one, and the
advantage recursion starts from zero beyond the rollout boundary.
Crucially, the bootstrap uses the next observation's regime value after a switch, and is suppressed on terminal transitions.
The advantages are centered and standardized over the entire rollout,
giving \(\widehat A_k\).

Writing
\(q_k(\theta)=\pi_\theta(u_k\mid o_k)/
\pi_{\theta^{\mathrm{old}}}(u_k\mid o_k)\),
the minimized actor and value losses are
\begin{align*}
 L_\pi=&-\E_{\mathrm{batch}}\!\left[
 \min\{q_k(\theta)\widehat A_k,
        \rm{clip}(q_k(\theta),0.8,1.2)\widehat A_k\}\right]
 -0.003\,\E_{\mathrm{batch}}[\mathcal H_G(o_k)], \\
 L_V=&\tfrac12\E_{\mathrm{batch}}\!
       \left[(V_{\phi_{I_k}}(o_k)-R_k)^2\right]. 
\end{align*}
The likelihood calculation includes the \(\tanh\) Jacobian; the fixed affine
action-scale factor cancels in the ratio. The entropy term is specifically
that of the latent Gaussian,
\(\mathcal H_G=\sum_{j=1}^3[\ell_j+\tfrac12\log(2\pi e)]\), not the
entropy of the bounded physical controls. It is a training regularizer
and is excluded from reported objective values.

\begin{table}[!htbp]
\centering\small
\caption{Hawkes-PPO training configuration.}
\label{tab:ppo-settings}
\begin{tabular}{@{}p{0.32\linewidth}p{0.63\linewidth}@{}}
\toprule
Setting & Choice\\ \midrule
Regime networks & Actor \(46\!-\!128\!-\!128\!-\!6\);
value \(46\!-\!128\!-\!128\!-\!1\), independently for each regime.\\
Initialization & Orthogonal weights, hidden gain \(\sqrt2\);
actor output gain \(0.01\), value output gain 1;
log-standard-deviation bias \(-0.7\).
Switching-mean bias \(\operatorname{atanh}(-0.95)\);
initial \(a,e\approx0.5\), \(\nu\approx0.1\).\\
Parallel sampling & 1,024 environments, exactly 512 AI starts and
512 Human starts. Each reset restores the assigned initial regime.\\
Rollout and update & 32 decisions per environment; four shuffled epochs;
minibatch 4,096; Adam; actor/value gradient norm caps 1 and 5.\\
Budget & 480 rollouts, 15,728,640 transitions,
15,360 actor and 15,360 value updates. No replay buffer or target-network
soft updates.\\
Initial learning rates & Actor \(3\cdot10^{-4}\), value \(5\cdot10^{-4}\);
common schedule below.\\
Selection & Every 786,432 transitions, 1,024 development paths per
initial regime; maximize the equally weighted mean of both start values.
Selected checkpoint at 15,728,640 transitions.\\
Seeds & Fit index 1; initialization 760001000;
training environment 760000001; development 761000001.\\
\bottomrule
\end{tabular}
\end{table}

Both learning rates are multiplied by the same schedule \(c(n)\), where
\(n\) counts collected environment transitions:
\begin{equation*}
 c(n)=
 \begin{cases}
 1, & n\leq n_0,\\
 0.1+0.45\!\left[1+\cos\!\left(\pi
            \dfrac{n-n_0}{n_1-n_0}\right)\right],&n_0<n<n_1,\\
 0.1,&n\geq n_1,
 \end{cases}
 \quad n_0=3{,}932{,}160,\quad n_1=11{,}796{,}480.
 \label{app:ppo-lr}
\end{equation*}
The small switching initialization changes only the two
switching-mean output biases; it does not remove Gaussian exploration.
It therefore initializes the deterministic action near \(\nu=0.1\),
rather than forcing all sampled switching intensities to equal \(0.1\).

At development and final evaluation, the Gaussian mean replaces the sampled
latent action in \eqref{app:ppo-policy}. Thus controls are deterministic
functions of current observations, but switches remain stochastic at the
selected intensity. The same frozen checkpoint is evaluated from both
starts. Training uses fresh on-policy rollouts, without oracle labels,
saved teacher paths, arbitrary-state reset data, or additional improvement
rounds outside the stated budget.

\begin{algorithm}[!htbp]
\caption{Hawkes-PPO with balanced starts and regime-specific networks}
\label{alg:hawkes-ppo-detailed}
\small
\algrenewcommand{\algorithmicrequire}{\textbf{Input:}}
\algrenewcommand{\algorithmicensure}{\textbf{Output:}}
\begin{algorithmic}[1]
\Require Simulator and public observation map; $M=1024$ environments, $L=32$ steps,
$K=480$ rollouts, $K_{\rm ep}=4$ epochs, minibatch size $B=4096$.
\Require Policy/value architectures, loss coefficients, and learning-rate schedule in Table~\ref{tab:ppo-settings}.
\Ensure Selected parameters $\theta^\star=(\theta_0^\star,\theta_1^\star)$ for deterministic deployment.
\State Initialize $(\theta_0,\theta_1)$ and $(\phi_0,\phi_1)$ independently; set switching mean biases to $\operatorname{atanh}(-0.95)$.
\State Assign $M/2$ environments to each initial regime; reset $X=R=Z=N^S=N^A=0$.
\State $J^\star\gets-\infty$; $\theta^\star\gets\theta$.
\For{$n=1,\ldots,K$}
  \State $(\theta_{\rm old},\phi_{\rm old})\gets(\theta,\phi)$; initialize an empty rollout buffer $\mathcal D$.
  \For{$k=0,\ldots,L-1$ \textbf{in all $M$ environments in parallel}}
    \State Observe $o_k=(t_k,X_k,I_k,Z_k,N_k^S,N_k^A,\mathbf1_{\{I_k=1\}}\widehat a^{\rm raw})$.
    \State Sample $z_k$ from the Gaussian actor $\theta_{{\rm old},I_k}$; set $u_k=\mathcal T(z_k)$.
    \State Hold $u_k$ for one decision interval; advance five simulator microsteps.
    \State Observe $(r_k,o_{k+1},d_k)$; set $\widetilde r_k=r_k/50$.
    \State Store transition, $z_k$, old log probability and old value predictions in $\mathcal D$.
    \State Reset terminated environments to their assigned initial regimes.
  \EndFor
  \State Compute $\delta_k$, $A_k$ and $R_k$ by terminal-masked GAE with $\gamma=1$, $\lambda_{\rm GAE}=0.95$.
  \State Standardize $A_k$ over $\mathcal D$ to obtain $\widehat A_k$; freeze $R_k$ and old predictions.
  \State Set $\eta_\pi,\eta_V$ using the transition-count schedule at $nML$.
  \For{$j=1,\ldots,K_{\rm ep}$}
    \For{each shuffled minibatch $\mathcal B\subset\mathcal D$ of size $B$}
      \State Route each sample through its current-regime actor and value network.
      \State $\theta\gets\operatorname{AdamStep}(\theta,\nabla_\theta L_{\rm actor}(\mathcal B),\eta_\pi;\,\text{norm cap }1)$.
      \State $\phi\gets\operatorname{AdamStep}(\phi,\nabla_\phi L_{\rm value}(\mathcal B),\eta_V;\,\text{norm cap }5)$.
    \EndFor
  \EndFor
  \If{$nML$ is a multiple of $786432$}
    \State Evaluate $u=\mathcal T(\mu_{\theta_I}(o))$ on 1024 development paths per initial regime.
    \State $\widehat J\gets\tfrac12(\widehat J_{\rm AI}+\widehat J_{\rm Human})$.
    \If{$\widehat J>J^\star$}
      \State $(J^\star,\theta^\star)\gets(\widehat J,\operatorname{copy}(\theta))$.
    \EndIf
  \EndIf
\EndFor
\State \Return $\theta^\star$; report separate-start outcomes on 4096 fresh common-noise paths per start.
\end{algorithmic}
\end{algorithm}

\newpage
\subsection{Full-Information Switching Oracle}
\label{app:oracle}
The oracle knows the model coefficients, the frozen map $a^*(t,X_t)$ and the
current recommendation error in either regime. Write its value as
$V_i(t,x,z,\zeta)$, where $\zeta$ represents the process $R$. The additional coordinate describes
oracle information and is not an extra observation supplied to RL. For $j\in\{S,A\}$, let $\mathbf e^{\,j}$ denote the filter
increment with components
$(\mathbf e^{\,j})^{\ell,k}=\mathbf 1_{\{k=j\}}$.
With $\alpha_0=a$, $\alpha_1={\rm clip}(a^*(t,x)+\zeta,0,1)$, the controlled generator, excluding regime switching, of a test function $F(t,x,z,\zeta)$ is
\begin{align*}
\mathcal L_i^{a,e}F(t,x,z,\zeta)
={}&(\alpha_i-0.18x)F_x-2\zeta F_\zeta
-\sum_{\ell,j}\beta_\ell z^{\ell,j}F_{z^{\ell,j}}
\nonumber\\
&+\frac12\bigl[\sigma(\alpha_i)^2+\gamma(L_i)^2\bigr]F_{xx}
+2s(L_i)^2F_{\zeta\zeta}
\nonumber\\
&+\lambda_i^S\,\mathbb E_{M^S}\!\left[
F(t,x-M^S,z+\mathbf e^{\,S},\zeta)
-F(t,x,z,\zeta)
\right]
\nonumber\\
&+\lambda_i^A\,\mathbb E_{M^A}\!\left[
F(t,x,z+\mathbf e^{\,A},\zeta-M^A)
-F(t,x,z,\zeta)
\right].
\end{align*}

Let $\overline m_{S,2}=\E[(M^S)^2]=0.0858333333333$ and
$q_i=\sigma(\alpha_i)^2+\gamma(L_i)^2+\lambda_i^S\overline m_{S,2}$.
The coupled HJB equations are
\begin{align*}
 0={}&\partial_t V_i+
 \sup_{a,e\in[0,1],\,\nu\in[0,4]}
 \Big\{\mathcal L_i^{a,e}V_i+f(x)-c_i-25q_i
       +\nu(V_{1-i}-V_i-0.20)-0.15\nu^2\Big\},
 \label{app:oracle-hjb}\\
 & V_i(8,x,z,\zeta)=g(x),\qquad i\in\{0,1\}.\nonumber
\end{align*}
The $\lambda_i^S\overline m_{S,2}$ term is the compensator of the jump-QV
penalty; jump losses in $V_i$ are separately represented by the generator.
Switching changes the regime at the same $(x,z,\zeta)$, with no recommendation
redraw or error reset. Optimization over $a,e$ is vacuous in AI mode in the
continuous problem. The switching maximizer is
\begin{equation*}
 \nu_i^*={\rm clip}\!\left(\frac{V_{1-i}-V_i-0.20}{0.30},0,4\right).
 \label{app:oracle-nu}
\end{equation*}
For Human project control, maximize
$aV_{0,x}+(\tfrac12V_{0,xx}-25)(0.16+0.04a)^2-0.75a^2$ on $[0,1]$.
The implementation compares endpoints and, when the quadratic is concave,
the clipped stationary candidate
\[
 a_{\rm stat}=\frac{V_{0,x}+0.0064V_{0,xx}-0.32}
                    {1.58-0.0016V_{0,xx}}.
\]
Mitigation $e$ is maximized over the grid $0,0.1,\ldots,1$, followed by local
candidate refinements of size $0.05$, $0.025$, and $0.0125$.
This finite search approximates the continuous mitigation maximum.

\paragraph{DGM implementation.}
Two independently parameterized gated DGM networks \cite{dgm}
represent the two regimes, each with width 64 and one gated block.
Exactly zero-weight filters are omitted, leaving ten active filter coordinates;
the input is $(t,x,\zeta,z_{\rm active})\in\mathbb R^{13}$.
The ansatz
\[
 V_i^\omega(t,x,z,\zeta)=g(x)+(8-t)\{-8+F_i^\omega(t,x,z,\zeta)\}
\]
enforces the terminal condition exactly. First and second derivatives are
obtained by automatic differentiation. Uniform-mark expectations use three-point
Gauss--Legendre quadrature in training and five points in residual audits.
The loss is the collocation average of
$\frac12\sum_{i=0}^1(\mathcal R_i/10)^2$, where $\mathcal R_i$ is the
HJB residual with greedily selected controls. No RL labels or temporal Bellman
targets are used for DGM training.

Training uses 6000 Adam updates with fresh batches of 512 collocation states,
gradient-norm cap 10, and learning rate
$10^{-3}[0.1+0.9\{1+\cos(\pi n/6000)\}/2]$.
A time curriculum samples $t$ uniformly from $[8-H_n,8]$, with $12\%$ of
samples placed at the left endpoint and
$H_n=\min\{8,0.5+7.5n/2100\}$.
Synthetic event histories generate physically interpretable filter states:
conditional on rates sampled uniformly in $[0.15,1.25]$, per-channel counts
are Poisson with mean rate times $t$ and are capped at 32 for collocation.
Ages are uniform in $[0,t]$, with $15\%$ of histories shifted toward recent
events by the map $r\mapsto r^2/\max(t,0.01)$; $15\%$ of filter samples are set to zero.
The project sample is
$0.6(1-e^{-0.18t})+0.7\sqrt{\max(t,0.02)}\xi$, $\xi\sim\mathcal N(0,1)$,
with $20\%$ replaced by $\mathcal U[-5,6]$.
Error samples are $-0.6\sum_n e^{-2r_n^A}+0.07\xi_E$, with $15\%$ replaced
by $\mathcal U[-2,1.3]$. \textcolor{black}{These distributions specify the HJB collocation samples. Physical-state dynamics and RL observation inputs retain their respective specifications.}

Every 1000 updates, a candidate policy is evaluated on 512 development paths
per initial regime, with equal-weight mean return used for selection.
Initialization seed is 730000001 and development seed is 731000000.
The selected checkpoint is update 3000, although all 6000 updates were completed.
Its held-out collocation residual RMSEs are $0.1530$ in Human and $0.5321$ in AI;
these describe the selected weights rather than a different residual-selected
checkpoint. \textcolor{black}{They quantify the HJB residual of that selected checkpoint.}

The oracle controls are deployed in the common fixed-grid simulator from
Section~\ref{app:choice}. Consequently, its continuous HJB value and its
simulated return need not coincide exactly: controls are committed for a decision
interval, and coefficients are frozen on microsteps. Reported objectives are
\textcolor{black}{Monte Carlo returns of the deployed policies. We use ``oracle'' to mean an approximate known-model, full-information numerical benchmark.}

% Source: human.tex
\subsection{Optimal Fixed-Human Benchmark: pure human mode design}
\label{app:human}
This benchmark solves the control problem with $I_t\equiv0$ and $\nu_t\equiv0$,
while optimizing both $a_t$ and $e_t$. It tests the benefit of switching against
an optimized Human policy, rather than against constant Human actions.
The numerical solution is an approximation to this optimal fixed-regime problem.

The recommendation error does not affect the Human project's coefficients or
intensities, and there is no future AI delegation. Thus the sufficient state is
$(t,x,z)$, with value $W(t,x,z)$.
Action-corruption events still matter through their excitation of future events.Using the filter increments $\mathbf e^{\,S}$ and $\mathbf e^{\,A}$
defined above, the fixed-Human generator is
\begin{align*}
\mathcal L_H^{a,e}W(t,x,z)
={}&(a-0.18x)W_x
-\sum_{\ell,j}\beta_\ell z^{\ell,j}W_{z^{\ell,j}}
\nonumber\\
&+\frac12\bigl[\sigma(a)^2+\gamma(L_0)^2\bigr]W_{xx}
\nonumber\\
&+\lambda_0^S\,\mathbb E_{M^S}\!\left[
W(t,x-M^S,z+\mathbf e^{\,S})-W(t,x,z)
\right]
\nonumber\\
&+\lambda_0^A\left[
W(t,x,z+\mathbf e^{\,A})-W(t,x,z)
\right].
\end{align*}
Here $\lambda_0^j=\lambda_0^j(z,e)$,
$L_0=\lambda_0^S+\lambda_0^A$, and all derivatives
are evaluated at $(t,x,z)$.
The fixed-Human HJB is
\begin{align*}
 0={}&W_t+\sup_{a,e\in[0,1]}
 \Big\{\mathcal L_H^{a,e}W+f(x)-3.25-0.75a^2-1.25e^2\nonumber\\
 &\hspace{32mm}-25[\sigma(a)^2+\gamma(L_0)^2
                  +\lambda_0^S\overline m_{S,2}]\Big\},
 \qquad W(8,x,z)=g(x).
\end{align*}
The action-event jump term remains even though no recommendation is implemented.

A single width-64, one-block gated DGM network uses the 12 inputs
$(t,x,z_{\rm active})$ and the hard-terminal ansatz
$W^\omega=g(x)+(8-t)(-12+F_H^\omega)$.
Action maximization, mitigation search, mark quadrature, Adam learning-rate
schedule, gradient cap and time/filter collocation scheme are the same as for
the switching oracle. The project collocation distribution is
$2.5(1-e^{-0.18t})+0.6\sqrt{\max(t,0.02)}\xi$, with $20\%$ replaced by
$\mathcal U[-5,6]$; no error-coordinate samples are needed.
The loss is $\E_{\rm colloc}[(\mathcal R_H/10)^2]$.
Training completes 6000 updates with batch size 512 and initialization seed
730000017. Every 1000 updates, selection uses 512 Human-start development
paths with seed 731000000. The selected update is 2000, with held-out residual
RMSE $0.02234$ and on-development-path residual RMSE $0.03286$.
The final independent Human-start objective is
$-96.027$ with $90\%$ interval $[-96.191,-95.864]$.
\textcolor{black}{As with the switching oracle, the reported quantity is the Monte Carlo return of an approximate HJB-derived policy deployed in the common simulator.}

% Source: other_rl.tex
\subsection{Additional RL Baselines}
\label{app:other-rl}

We train PPO, SAC \cite{haarnoja2018sac}, DDPG \cite{lillicrap2015continuous} and CT-DDPG \cite{cheng2026deterministicpolicygradientreinforcement, bielecki2026continuoustimereinforcementlearningcontrolled} both with and without the same 40 Hawkes
filters. The prefix ``Hawkes-'' denotes the filtered version. All receive
$(t,X,I,N^S,N^A,\mathbf{1}_{\{I=1\}}\widehat a^{\rm raw})$; only the filtered versions receive
$Z$. The unfiltered implementations set the 40 filter coordinates to zero,
retaining the same network dimensions. Neither version receives $R$, the
nominal recommendation $a^*$, the current recommendation in Human mode,
or oracle labels. Every method uses separate networks for the current regime, rather than separate policies for the initial regime. All MLPs have
two hidden layers of width 128 with SiLU activations. Actors have 46 inputs;
SAC/DDPG action-value critics and CT-DDPG rate networks have 49 inputs.
SAC has two action-value critics per regime, DDPG one, and CT-DDPG one value
and one rate network per regime. Native PPO/SAC/DDPG heads are independently
initialized; the two CT-DDPG regime heads start as identical copies with
subsequently separate parameters. CT-DDPG initially sets input weights on
filter coordinates to zero, and learns these weights during training.

All fits use 1,024 parallel environments, half assigned to each standard
initial regime, with $\Delta=0.05$, simulation step $0.01$, undiscounted
returns and reward divisor 50. Initial deterministic switching centers are
approximately $0.1$ in both regimes. Initial Human action/effort centers are
$(0.5,0.5)$ for PPO/SAC/DDPG and $(0.35,0.45)$ for CT-DDPG. Exploration is
still active: a small deterministic center does not mean that executed
switching rates remain close to $0.1$. PPO without filters otherwise uses
exactly the Hawkes-PPO procedure, including its pre-tanh Gaussian entropy
coefficient $0.003$ and 480 rollouts of 32 decisions, each with four update
epochs and minibatch size 4,096.

\paragraph{SAC and DDPG:} Let $v\in[-1,1]^3$ denote normalized actions, mapped to physical controls by
$u=((v_1+1)/2,(v_2+1)/2,2(v_3+1))$. Write $\widetilde r=r/50$ and let $d$
indicate episode termination. SAC samples a tanh-squashed diagonal Gaussian,
uses twin critics, and minimizes
\begin{align*}
 y_{\rm SAC}&=\widetilde r+(1-d)\left[\min_{j=1,2}\bar Q_j(o',v')
                    -\alpha\log\pi_\theta(v'\mid o')\right],
 v'\sim\pi_\theta(\cdot\mid o'),\\
 L_Q&=\sum_{j=1}^2\mathbb E[(Q_j(o,v)-y_{\rm SAC})^2],
 L_\pi=\mathbb E[\alpha\log\pi_\theta(v\mid o)-\min_jQ_j(o,v)].
\end{align*}
Here the entropy is the differential entropy of the squashed
normalized action distribution, including the tanh Jacobian, rather than
PPO's pre-tanh Gaussian entropy. Log standard deviations are bounded by
$[-5,1]$ and initialized at $-0.7$. The temperature starts at $\alpha=0.01$
and is learned with target entropy $-3$ via
\[
 L_\alpha=-\mathbb E\!\left[\log\alpha\,
          \bigl(\operatorname{stopgrad}\log\pi_\theta(v\mid o)-3\bigr)\right].
\]
Its Adam learning rate is $3\times10^{-4}$, without decay, and
$\alpha\in[10^{-5},0.2]$ is enforced after each update. The selected
unfiltered and filtered SAC checkpoints have temperatures
$1.1844\times10^{-4}$ and $2.0036\times10^{-4}$, respectively.
Entropy affects training only; all reported objectives use the original
control criterion, and evaluation uses the transformed Gaussian mean.

DDPG instead uses one deterministic actor and one action-value critic per
regime, with targets
$y_{\rm DDPG}=\widetilde r+(1-d)\bar Q(o',\bar\mu(o'))$,
critic loss $\mathbb E[(Q-y_{\rm DDPG})^2]$, and actor loss
$-\mathbb E[Q(o,\mu(o))]$. Both off-policy methods store 262,144 transitions,
collect 32 initial vector steps of uniform actions, and then perform two
minibatch updates per vector step, with batch size 4,096. DDPG adds clipped
Gaussian exploration in normalized coordinates, with standard deviation
decreasing linearly from $0.25$ to $0.05$ over the first 80\% of training.
SAC explores by sampling its actor. Polyak updates use
$\bar\theta\leftarrow(1-\tau)\bar\theta+\tau\theta$ after each update:
SAC updates its two target critics, while DDPG updates its target critic
and actor.

\paragraph{Implemented CT-DDPG extension.}
Following the value/advantage-rate approach of CT-DDPG \cite{cheng2026deterministicpolicygradientreinforcement,bielecki2026continuoustimereinforcementlearningcontrolled}, this baseline uses an integrated martingale residual and an advantage-rate
network, rather than a discrete-time action-value critic. Let $q_\psi(o,u)$
be its raw rate network and $\mu_\theta(o)$ its actor. The centered rate is
\[
 A_{\psi,\theta}(o,u)=q_\psi(o,u)-q_\psi(o,\mu_\theta(o)).
\]
For a replay segment of $\ell\in\{4,\ldots,20\}$ decisions, sampled uniformly,
the joint value/rate update minimizes the squared residual
\begin{equation}
 \mathbb E\!\left[
 \left\{\bar V(o_{k+\ell})-V(o_k)+
 \sum_{j=k}^{k+\ell-1}
 \left(\frac{r_j^{\rm run}}{50}-\Delta A_{\psi,\theta}(o_j,u_j)\right)
 \right\}^{\!2}\right].\label{app:ct-residual}
\end{equation}
The reference actor is held fixed during this critic update. The reward
$r_j^{\rm run}$ includes running, QV and switching contributions, with the
terminal payoff removed once. Values enforce the known terminal condition
through $V(o)=g(x)/50+(T-t)v_\phi(o)$. Actor updates maximize the
uncentered rate, using loss
$-\Delta\mathbb E[q_\psi(o,\mu_\theta(o))]$. Only the value network has a
Polyak target; there is no target actor or target rate network.

% Source: all_results.tex
\paragraph{Returns under the common training budget.}
Table~\ref{app:all-values} reports the selected low-initial-intensity fits
using the common evaluation protocol. Under the common environment-interaction budget, Hawkes-PPO achieves
the highest evaluated objective among the RL methods from both initial
regimes, and it approaches the oracle well. The weaker performance of the alternative methods may partly
reflect optimization difficulties under the available training budget.
In particular, CT-DDPG selects relatively early checkpoints and
subsequently deteriorates, suggesting instability in the present
implementation. \textcolor{black}{The comparison concerns the selected fitted policies under the common budget. The reported confidence intervals quantify Monte Carlo evaluation uncertainty with the training seeds and fitted weights held fixed.}
\begin{table}[!htbp]
\centering\small
\caption{Objective values and $90\%$ Monte Carlo confidence intervals.
Each reported entry uses 4096 paths; a dash indicates an incompatible
initial regime for a fixed-regime benchmark. Higher is better.}
\label{app:all-values}
\setlength{\tabcolsep}{5pt}
\begin{tabular}{@{}lcc@{}}
\toprule
Method & AI start & Human start\\\midrule
Oracle (HJB/DGM) & -93.262 [-93.508, -93.015] & -94.230 [-94.468, -93.991]\\
Hawkes-PPO & -93.608 [-93.861, -93.354] & -94.676 [-94.915, -94.437]\\
Pure AI & -98.127 [-98.537, -97.718] & ---\\
Pure Human (HJB/DGM) & --- & -96.027 [-96.191, -95.864]\\
\textbf{Other RL methods:}\\
PPO & -94.954 [-95.198, -94.710] & -97.219 [-97.407, -97.031]\\
Hawkes-SAC & -95.750 [-96.033, -95.467] & -96.579 [-96.853, -96.305]\\
SAC & -94.759 [-95.029, -94.489] & -96.079 [-96.347, -95.810]\\
Hawkes-DDPG & -95.810 [-96.128, -95.492] & -96.816 [-97.128, -96.504]\\
DDPG & -95.591 [-95.876, -95.306] & -96.783 [-97.065, -96.502]\\
Hawkes-CT-DDPG & -98.776 [-99.155, -98.397] & -101.670 [-101.902, -101.438]\\
CT-DDPG & -98.984 [-99.357, -98.611] & -101.811 [-102.106, -101.516]\\
\bottomrule
\end{tabular}
\end{table}

\subsection{Comparative statics, supervision mechanisms,
and numerical sensitivity}
\label{app:supervision-statics}

This subsection provides additional supporting evidence for
Section~\ref{sec:supervision-statics}. The results are reported in Table \ref{tab:comparative-intervals}. 

\begingroup\color{black}
\begin{table}[H]
\color{black}\centering
\caption{Supervision and project outcomes across model configurations.}
\label{tab:comparative-statics}
\begingroup\fontsize{8}{10}\selectfont
\setlength{\tabcolsep}{4pt}
\begin{tabular}{lrrrrrrr}
\toprule
Configuration & \shortstack{Human time (\%)} & \shortstack{Effort (Human)} & Switches & Incidents & \shortstack{Burst (\%)} & \shortstack{Terminal quality} & Objective \\
\midrule
\multicolumn{8}{l}{\textit{Panel A. AI initial regime}} \\
\addlinespace[2pt]
\textbf{Baseline} & \textbf{46.4} & \textbf{0.896} & \textbf{1.46} & \textbf{5.40} & \textbf{45.9} & \textbf{1.81} & \textbf{-93.61} \\
\addlinespace[2pt]
Excitation 0.5 & 0.9 & 0.893 & 0.27 & 5.59 & 45.9 & 1.17 & -85.96 \\
Excitation 1.5 & 73.6 & 0.842 & 1.22 & 5.13 & 38.6 & 1.88 & -95.85 \\
\addlinespace[2pt]
Human cost 0.75 & 76.0 & 0.731 & 1.15 & 5.05 & 36.7 & 1.82 & -89.68 \\
Human cost 1.25 & 13.7 & 0.973 & 1.01 & 6.20 & 57.4 & 1.30 & -96.33 \\
\addlinespace[2pt]
Switch fee $\times0.5$ & 45.4 & 0.880 & 1.34 & 5.44 & 47.0 & 1.90 & -93.72 \\
Switch fee $\times2$ & 43.0 & 0.853 & 0.91 & 5.53 & 48.6 & 2.11 & -94.36 \\
\addlinespace[2pt]
A 0.5; cost 0.75 & 2.4 & 0.792 & 0.23 & 5.58 & 45.8 & 1.16 & -85.84 \\
A 0.5; cost 1.25 & 0.5 & 0.753 & 0.23 & 5.61 & 46.0 & 1.15 & -86.00 \\
A 1.5; cost 0.75 & 86.5 & 0.783 & 1.24 & 4.99 & 35.5 & 1.77 & -90.86 \\
A 1.5; cost 1.25 & 64.0 & 0.914 & 1.68 & 5.26 & 42.6 & 1.91 & -100.14 \\
\midrule
\multicolumn{8}{l}{\textit{Panel B. Human initial regime}} \\
\addlinespace[2pt]
\textbf{Baseline} & \textbf{58.4} & \textbf{0.809} & \textbf{2.01} & \textbf{5.28} & \textbf{43.4} & \textbf{1.88} & \textbf{-94.68} \\
\addlinespace[2pt]
Excitation 0.5 & 5.3 & 0.516 & 1.42 & 5.58 & 45.7 & 1.19 & -87.33 \\
Excitation 1.5 & 98.1 & 0.696 & 0.22 & 4.92 & 33.8 & 2.02 & -96.85 \\
\addlinespace[2pt]
Human cost 0.75 & 99.6 & 0.623 & 0.07 & 4.93 & 33.5 & 1.93 & -89.75 \\
Human cost 1.25 & 19.1 & 0.862 & 2.09 & 6.13 & 56.5 & 1.33 & -97.73 \\
\addlinespace[2pt]
Switch fee $\times0.5$ & 58.1 & 0.793 & 1.92 & 5.31 & 44.0 & 1.96 & -94.65 \\
Switch fee $\times2$ & 82.5 & 0.693 & 0.66 & 5.09 & 38.3 & 2.51 & -96.17 \\
\addlinespace[2pt]
A 0.5; cost 0.75 & 13.6 & 0.607 & 1.28 & 5.51 & 44.6 & 1.19 & -86.75 \\
A 0.5; cost 1.25 & 4.2 & 0.417 & 1.41 & 5.60 & 45.8 & 1.15 & -87.65 \\
A 1.5; cost 0.75 & 99.6 & 0.708 & 0.13 & 4.90 & 33.5 & 1.83 & -90.45 \\
A 1.5; cost 1.25 & 74.4 & 0.826 & 2.01 & 5.16 & 40.2 & 1.98 & -101.52 \\
\bottomrule
\end{tabular}
\endgroup
\par\smallskip
\begin{minipage}{0.98\textwidth}\footnotesize
4,096 evaluation paths per initial regime. Entries are path means, except effort, which is total Human effort divided by total Human time. Burst denotes at least three combined state and action incidents in a window of one time unit. Excitation, Human cost, and per-switch fee are baseline multipliers. 
\end{minipage}
\end{table}
\endgroup

\paragraph{Configuration design.}
Let $A^0$ denote the baseline excitation matrix, $k_0$ the baseline
Human fixed operating-cost coefficient, and $\chi$ the baseline
per-switch fee. We vary one factor at a time:
\[
\begin{array}{lll}
\text{Excitation strength:}
    & A=s_A A^0,
    & s_A\in\{0.5,1,1.5\},\\
\text{Human fixed cost:}
    & k_0\mapsto s_c k_0,
    & s_c\in\{0.75,1,1.25\},\\
\text{Per-switch fee:}
    & \chi\mapsto s_K \chi,
    & s_K\in\{0.5,1,2\}.
\end{array}
\]
The Human action and effort costs remain unchanged when $k_0$
varies. The quadratic switching-rate penalty remains unchanged
when the per-switch fee varies. The baseline fee is $\chi=0.2$,
so the three per-switch charges are $0.1$, $0.2$, and $0.4$;
the intensity-penalty coefficient stays at $0.15$. The baseline is shared across
the three comparisons, giving seven distinct one-factor
configurations. Crossing $s_A\in\{0.5,1.5\},
s_c\in\{0.75,1.25\}$
adds four configurations, for eleven in total.
All unlisted physical and cost parameters retain their
baseline values.

The intensity specification is
\[
\lambda_t^j=\mu_j+(1-\bar e_t)H_t^j,
\qquad
\mu=(0.35,0.25),
\]
where $\bar e_t$ is the effective mitigation effort defined in
the main model. In particular, effort does not suppress
the baseline intensity.\\

Actor and value learning rates start at
$3\times10^{-4}$ and $5\times10^{-4}$, remain constant through
$3,932,160$ transitions, and decay according to the original schedule
to one tenth by $11,796,480$ transitions. Independent fits run in
parallel without changing these per-fit settings or budgets.

\paragraph{Evaluation,  uncertainty and Monte Carlo.}
Each selected policy is evaluated on $4,096$ paths per initial
regime in one batch, using seed $783000000$ and the original numerical
grid. Common random streams are used across configurations, and
AI and Human starts are reported separately. The reused baseline
reproduces all twenty-three original recorded per-path measurements
exactly for both initial regimes.

For ordinary path means, pointwise $90\%$ Monte Carlo intervals
use
\[
\bar Y \pm z_{0.95}\frac{s_Y}{\sqrt n},
\qquad z_{0.95}\simeq1.64485.
\]
Marginal burst probabilities use Wilson intervals.
Endpoint contrasts use paired path differences, with paired
influence functions for conditional-effort ratios.
These intervals condition on the selected fitted policies.
They exclude variability across training seeds,
checkpoint-selection uncertainty, and discretization error,
and are not adjusted for multiple comparisons.

\paragraph{Reported measurements.}
For path $p$, define Human duration and Human effort exposure as
\[
T_{H,p}=\int_0^T\mathbf1_{\{I_t=0\}}\,dt,
\qquad
E_{H,p}=\int_0^T\mathbf1_{\{I_t=0\}}e_t\,dt.
\]
Human time and conditional effort are estimated by
\[
\widehat h
=
\frac{1}{nT}\sum_{p=1}^n T_{H,p},
\qquad
\widehat e_H
=
\frac{\sum_{p=1}^n E_{H,p}}
     {\sum_{p=1}^n T_{H,p}}.
\]
The table reports $100\widehat h$.
Total Human effort exposure is
$\sum_p E_{H,p}/(nT)=\widehat h\,\widehat e_H$,
which differs from effort conditional on Human operation.
The latter is a pooled ratio, not the average of
path-specific conditional ratios.
Its standard error uses the path influences
\[
\psi_p
=
\frac{E_{H,p}-\widehat e_H T_{H,p}}
     {n^{-1}\sum_{q=1}^nT_{H,q}},
\qquad
\widehat{\mathrm{se}}(\widehat e_H)
=
\frac{s_\psi}{\sqrt n}.
\]

Incidents are $N_T^S+N_T^A$; realized switches count accepted
regime changes. A burst occurs when at least three combined
incidents fall within any closed sliding window of one time
unit. Terminal quality is $X_T$, not $g(X_T)$.
Path volatility is $\sqrt{[X]_T/T}$ averaged across paths,
rather than the square root of mean quadratic variation.

\paragraph{Physical outcomes and the objective.}
At baseline costs, raising excitation from $0.5$ to $1.5$
reduces mean incidents from $5.595$ to $5.127$ for AI starts
and from $5.581$ to $4.918$ for Human starts, while increasing
terminal quality from $1.168$ to $1.881$ and from $1.189$
to $2.023$, respectively. These outcomes combine a changed
incident environment with the response of a separately fitted policy.

Raising Human fixed cost from $0.75$ to $1.25$ at baseline
excitation increases incidents from $5.054$ to $6.202$
for AI starts and from $4.926$ to $6.130$ for Human starts.
Burst probabilities increase from $36.7\%$ to $57.4\%$
and from $33.5\%$ to $56.5\%$, respectively.
Conditional effort rises, but reduced Human duration causes
total Human effort exposure to fall.
Increasing the per-switch fee multiplier from $0.5$ to $2$
reduces mean switches from $1.344$ to $0.907$ for AI starts and
from $1.924$ to $0.659$ for Human starts. Incident counts and
burst probabilities increase for AI starts and decrease for
Human starts, while terminal quality increases for both.

\paragraph{Reward and cost decomposition.}
We separate running and terminal project benefits, Human and AI
operating costs, switching-rate and realized-switch costs,
and the continuous and jump components of the
quadratic-variation penalty.
Their signed sum reconstructs the reported objective.
This decomposition distinguishes improvements obtained through
cost savings from improvements in physical project outcomes.
\textcolor{black}{In particular, cost savings can improve the objective even when incidents increase or terminal quality falls.}

\paragraph{Supervision around incidents and takeovers.}
We align baseline trajectories around incidents, candidate
cluster onsets, and actual AI-to-Human takeovers.
An incident may belong to either channel.
A candidate cluster onset is an incident followed by at least
two further incidents within one time unit; this is a
\textcolor{black}{retrospective event definition that permits overlapping candidate clusters.}

We retain complete windows on $[-1,1]$ within the recorded
trajectory. Within each path, the earliest eligible anchor is
retained and subsequent anchors must be at least two time
units apart.
Recorded left-constant microstep values are sampled on a
$0.05$-unit lag grid.
Retained windows are first averaged within each path;
eligible paths then receive equal weight.
Conditional effort and the AI-only takeover rate use ratios
of the corresponding path-averaged numerators and exposure
denominators, with path-level delta-method uncertainty.

\paragraph{Risk during periods selected for supervision.}
At baseline excitation, higher Human fixed cost is associated
with less Human time and higher conditional effort.
During Human operation, total unmitigated excitation can be
reconstructed as
\[
H_t^S+H_t^A
=
\frac{\lambda_t^S+\lambda_t^A-0.6}{1-e_t}.
\]
This diagnostic excludes microsteps with $e_t\geq0.999$.
At the high-cost setting, it excludes $23.02\%$ and $16.86\%$
of Human microsteps for AI and Human starts, compared with
$0.18\%$ and $0.12\%$ at baseline cost.
Among included Human microsteps, the mean excitation across
cost multipliers $0.75,1,1.25$ is $0.230,0.286,0.306$ for AI
starts and $0.188,0.243,0.235$ for Human starts.
The latter pattern is nonmonotonic, and both comparisons are
sensitive to the effort-dependent exclusions.

\paragraph{Joint changes in excitation and Human cost.}
The four crossed configurations test whether the one-factor
patterns persist when risk and oversight cost change together.
At both crossed cost levels, increasing excitation from $0.5$
to $1.5$ increases Human time and terminal quality and reduces
incident counts and burst probability for both initial regimes.
At both crossed excitation levels, increasing Human cost from
$0.75$ to $1.25$ reduces Human time and increases incidents and
burst probability. These Human-time and incident directions
therefore persist in the crossed checks.

Effort, switching, and terminal quality have exceptions.
Higher Human cost reduces conditional effort at low excitation
but increases it at high excitation, for both starts.
The AI-start effort response to stronger excitation is unresolved
at low Human cost, with difference $-0.008$ and paired $90\%$
interval $[-0.020,0.003]$, but becomes positive at high Human
cost, unlike the negative response at baseline cost.
At high excitation, higher Human cost increases terminal quality
by $0.133$ for AI starts and $0.147$ for Human starts, despite
higher incident burden; at low excitation, quality decreases.
\textcolor{black}{Thus the quality response to supervision cost depends on excitation strength.}

\textcolor{black}{Responses across intermediate configurations can be nonmonotone.}
With expensive supervision, AI-start incident counts across
excitation multipliers $0.5,1,1.5$ are $5.606,6.202,5.261$.
At high excitation, Human-start quality across cost multipliers
$0.75,1,1.25$ is $1.833,2.023,1.980$.
Switching responses likewise depend on the starting regime
and cost-risk combination.

\begingroup
\newcommand{\csCI}[3]{\shortstack[c]{#1\\{\fontsize{7}{8}\selectfont [#2,#3]}}}
\begin{table}[p]
\centering
\caption{Supervision and project outcomes under the original numerical protocol. Means appear above conditional 90\% evaluation-path Monte Carlo intervals.}
\label{tab:comparative-intervals}
\fontsize{8}{9}\selectfont
\setlength{\tabcolsep}{1.5pt}
\renewcommand{\arraystretch}{1.02}
\resizebox{\linewidth}{!}{%
\begin{tabular}{@{}lccccccc@{}}
\toprule
Configuration & \shortstack{Human\\time (\%)} & \shortstack{Conditional\\effort} & \shortstack{Switches\\/ path} & \shortstack{Incidents\\/ path} & \shortstack{Burst\\prob. (\%)} & \shortstack{Terminal\\quality} & Objective \\
\midrule
\multicolumn{8}{@{}l}{\textbf{Panel A: AI initial regime}} \\
\addlinespace[2pt]
\textbf{Baseline} & \csCI{46.436}{45.694}{47.179} & \csCI{0.896}{0.893}{0.898} & \csCI{1.457}{1.430}{1.484} & \csCI{5.396}{5.335}{5.458} & \csCI{45.874}{44.597}{47.157} & \csCI{1.811}{1.787}{1.835} & \csCI{-93.608}{-93.861}{-93.354} \\
\addlinespace[2pt]
Excitation 0.5 & \csCI{0.946}{0.857}{1.034} & \csCI{0.893}{0.882}{0.904} & \csCI{0.271}{0.251}{0.291} & \csCI{5.595}{5.523}{5.666} & \csCI{45.923}{44.645}{47.206} & \csCI{1.168}{1.143}{1.193} & \csCI{-85.962}{-86.204}{-85.720} \\
Excitation 1.5 & \csCI{73.607}{73.045}{74.168} & \csCI{0.842}{0.838}{0.846} & \csCI{1.222}{1.204}{1.241} & \csCI{5.127}{5.068}{5.187} & \csCI{38.647}{37.404}{39.906} & \csCI{1.881}{1.856}{1.905} & \csCI{-95.849}{-96.069}{-95.628} \\
\addlinespace[2pt]
Human cost 0.75 & \csCI{75.998}{75.463}{76.533} & \csCI{0.731}{0.726}{0.735} & \csCI{1.151}{1.135}{1.166} & \csCI{5.054}{4.996}{5.113} & \csCI{36.670}{35.441}{37.917} & \csCI{1.817}{1.793}{1.841} & \csCI{-89.677}{-89.868}{-89.487} \\
Human cost 1.25 & \csCI{13.716}{13.237}{14.196} & \csCI{0.973}{0.972}{0.974} & \csCI{1.013}{0.981}{1.045} & \csCI{6.202}{6.126}{6.277} & \csCI{57.446}{56.171}{58.712} & \csCI{1.297}{1.271}{1.323} & \csCI{-96.334}{-96.669}{-96.000} \\
\addlinespace[2pt]
Switch fee $\times0.5$ & \csCI{45.382}{44.635}{46.130} & \csCI{0.880}{0.877}{0.883} & \csCI{1.344}{1.319}{1.370} & \csCI{5.440}{5.378}{5.502} & \csCI{47.021}{45.741}{48.306} & \csCI{1.899}{1.874}{1.924} & \csCI{-93.720}{-93.978}{-93.462} \\
Switch fee $\times2$ & \csCI{43.028}{42.200}{43.856} & \csCI{0.853}{0.849}{0.857} & \csCI{0.907}{0.890}{0.924} & \csCI{5.533}{5.469}{5.597} & \csCI{48.633}{47.350}{49.918} & \csCI{2.109}{2.081}{2.137} & \csCI{-94.360}{-94.627}{-94.092} \\
\addlinespace[2pt]
A 0.5; cost 0.75 & \csCI{2.384}{2.172}{2.595} & \csCI{0.792}{0.781}{0.802} & \csCI{0.231}{0.214}{0.248} & \csCI{5.580}{5.508}{5.651} & \csCI{45.801}{44.523}{47.084} & \csCI{1.156}{1.131}{1.181} & \csCI{-85.839}{-86.078}{-85.599} \\
A 0.5; cost 1.25 & \csCI{0.450}{0.402}{0.499} & \csCI{0.753}{0.734}{0.772} & \csCI{0.230}{0.211}{0.249} & \csCI{5.606}{5.534}{5.678} & \csCI{45.996}{44.718}{47.279} & \csCI{1.148}{1.123}{1.173} & \csCI{-85.996}{-86.240}{-85.753} \\
A 1.5; cost 0.75 & \csCI{86.481}{86.165}{86.797} & \csCI{0.783}{0.779}{0.788} & \csCI{1.236}{1.217}{1.255} & \csCI{4.985}{4.927}{5.044} & \csCI{35.498}{34.278}{36.737} & \csCI{1.773}{1.749}{1.797} & \csCI{-90.862}{-91.045}{-90.679} \\
A 1.5; cost 1.25 & \csCI{64.031}{63.394}{64.668} & \csCI{0.914}{0.912}{0.916} & \csCI{1.679}{1.650}{1.709} & \csCI{5.261}{5.201}{5.321} & \csCI{42.578}{41.313}{43.853} & \csCI{1.906}{1.881}{1.930} & \csCI{-100.143}{-100.409}{-99.877} \\
\midrule
\multicolumn{8}{@{}l}{\textbf{Panel B: Human initial regime}} \\
\addlinespace[2pt]
\textbf{Baseline} & \csCI{58.424}{57.654}{59.194} & \csCI{0.809}{0.806}{0.812} & \csCI{2.013}{1.981}{2.044} & \csCI{5.283}{5.223}{5.343} & \csCI{43.384}{42.115}{44.661} & \csCI{1.879}{1.854}{1.903} & \csCI{-94.676}{-94.915}{-94.437} \\
\addlinespace[2pt]
Excitation 0.5 & \csCI{5.334}{5.190}{5.478} & \csCI{0.516}{0.507}{0.525} & \csCI{1.422}{1.397}{1.447} & \csCI{5.581}{5.509}{5.652} & \csCI{45.654}{44.377}{46.937} & \csCI{1.189}{1.164}{1.214} & \csCI{-87.328}{-87.570}{-87.085} \\
Excitation 1.5 & \csCI{98.061}{97.855}{98.266} & \csCI{0.696}{0.691}{0.700} & \csCI{0.222}{0.204}{0.240} & \csCI{4.918}{4.861}{4.975} & \csCI{33.813}{32.609}{35.040} & \csCI{2.023}{1.999}{2.047} & \csCI{-96.849}{-97.017}{-96.682} \\
\addlinespace[2pt]
Human cost 0.75 & \csCI{99.594}{99.510}{99.679} & \csCI{0.623}{0.619}{0.628} & \csCI{0.070}{0.060}{0.081} & \csCI{4.926}{4.868}{4.983} & \csCI{33.521}{32.318}{34.744} & \csCI{1.934}{1.911}{1.958} & \csCI{-89.746}{-89.910}{-89.581} \\
Human cost 1.25 & \csCI{19.065}{18.554}{19.577} & \csCI{0.862}{0.858}{0.866} & \csCI{2.088}{2.053}{2.122} & \csCI{6.130}{6.056}{6.205} & \csCI{56.470}{55.192}{57.739} & \csCI{1.333}{1.307}{1.359} & \csCI{-97.731}{-98.061}{-97.401} \\
\addlinespace[2pt]
Switch fee $\times0.5$ & \csCI{58.076}{57.288}{58.865} & \csCI{0.793}{0.790}{0.796} & \csCI{1.924}{1.894}{1.954} & \csCI{5.310}{5.250}{5.370} & \csCI{44.043}{42.771}{45.322} & \csCI{1.965}{1.940}{1.989} & \csCI{-94.649}{-94.892}{-94.407} \\
Switch fee $\times2$ & \csCI{82.461}{81.736}{83.186} & \csCI{0.693}{0.688}{0.697} & \csCI{0.659}{0.636}{0.683} & \csCI{5.087}{5.029}{5.145} & \csCI{38.330}{37.089}{39.587} & \csCI{2.513}{2.486}{2.540} & \csCI{-96.174}{-96.380}{-95.968} \\
\addlinespace[2pt]
A 0.5; cost 0.75 & \csCI{13.579}{13.168}{13.989} & \csCI{0.607}{0.601}{0.614} & \csCI{1.277}{1.258}{1.295} & \csCI{5.506}{5.437}{5.576} & \csCI{44.580}{43.307}{45.861} & \csCI{1.187}{1.162}{1.212} & \csCI{-86.750}{-86.983}{-86.518} \\
A 0.5; cost 1.25 & \csCI{4.214}{4.109}{4.319} & \csCI{0.417}{0.411}{0.424} & \csCI{1.409}{1.384}{1.433} & \csCI{5.597}{5.525}{5.669} & \csCI{45.850}{44.572}{47.133} & \csCI{1.154}{1.129}{1.179} & \csCI{-87.649}{-87.895}{-87.404} \\
A 1.5; cost 0.75 & \csCI{99.572}{99.499}{99.644} & \csCI{0.708}{0.704}{0.713} & \csCI{0.133}{0.118}{0.147} & \csCI{4.896}{4.839}{4.953} & \csCI{33.521}{32.318}{34.744} & \csCI{1.833}{1.809}{1.856} & \csCI{-90.453}{-90.617}{-90.289} \\
A 1.5; cost 1.25 & \csCI{74.429}{73.792}{75.066} & \csCI{0.826}{0.823}{0.829} & \csCI{2.014}{1.978}{2.050} & \csCI{5.160}{5.102}{5.218} & \csCI{40.234}{38.981}{41.501} & \csCI{1.980}{1.955}{2.004} & \csCI{-101.517}{-101.763}{-101.271} \\
\bottomrule
\end{tabular}%
}
\par\smallskip
\begin{minipage}{\textwidth}\fontsize{7}{8}\selectfont
\end{minipage}
\end{table}
\endgroup

\clearpage
\begingroup\color{black}
\section{Data, event construction, and empirical history dependence}
\label{app:empirical-history}
We ask whether recorded error history improves prediction beyond observed
background-risk differences. AgentLogs supplies the primary continuous-time,
two-type analysis; ODCV-Bench supplies complementary binary occurrence evidence
in decision time. This appendix explains the observations, the fitted models and
what their likelihood gains mean. The main text reports only the $\rho=30$ AgentLogs fit and $90\%$ intervals; the additional audit results below preserve the exploratory record. \textcolor{black}{These analyses assess the predictive contribution of the specified incident-history representations.}

\subsection{AgentLogs: expanded data and event construction}
\label{app:history-data}
AgentLogs contains GitHub cloud-agent execution traces
\cite{richards2026AgentLogsDatasetOpening}. We use release
\texttt{1.0-20261008}, schema v1.0, focusing on five repositories:
Azure SDK for Java, EduMIPS64, bnd, c:geo and Elasticsearch.
Two downloaded shards cover 1,238 of 1,249 metadata sessions in these repositories;
\textcolor{black}{the sampling frame is these five repositories and the downloaded shards.}
Beyond earlier reviewed data, task-based sampling selected 258 additional
sessions without selection on error outcomes. Structural and temporal checks
left 209 sessions for contextual review, covering 9,381 action records.
Eighty-six newly reviewed sessions met the complete-sequence requirements.
Combined with 77 earlier eligible sessions, this yields 163 sessions and
161 error-bearing decisions over 1,438.589 minutes.

An \emph{action error} is an operationally invalid action or violation of a
necessary precondition. A \emph{state error} is a factual assertion about the
project or environment contradicted by trace evidence. \textcolor{black}{Both labels require contextual evidence of the defined error; exit codes, diagnostic test failures and task outcomes are interpreted in that context.} \textcolor{black}{Labels come from contextual assistant review with focused consistency checks.} Whole sessions are
excluded when unresolved actions or claims prevent a complete event sequence;
unknown intervals are not treated as error-free. This conservative selection
limits generalization.

Streaming and telemetry duplicates are collapsed, and parallel decision batches
are preserved. Each newly generated qualifying error counts, including a
repeated mistake. Multiple errors of the same type within one decision form
one event. A batch containing both types would require a common-jump observation
model and is excluded from the simple two-type fit rather than artificially
ordered; none remains in the final eligible cohort. There are 50 action and
111 state events in total.

\paragraph{What the timestamps mean.}
Time is elapsed minutes from metadata creation to completion, including setup
and idle time. An event is assigned the recorded proposal/occurrence time of
its originating output, even if its label is verified later. Source occurrence
timestamps have whole-second resolution. We neither jitter nor repair them;
unsuitable timing, ambiguous reconstruction and unresolved labels are audited.
Filters reset to zero per session. \textcolor{black}{The analysis relies on archived log coverage. Retrospective labels use evidence that may become available after an incident's recorded occurrence time; deployment would require an explicit detection process.}

\paragraph{Example: action and state events.}
In retained Azure session \texttt{0645147d\ldots}, entry 51 tries to read a
script through an incorrect path containing an extra \texttt{common} directory.
Earlier search identifies the correct location and the requested read fails:
this is an action event at elapsed time 5.3162 minutes. Entry 156 later reports
that authenticated npm setup is missing, despite inspected source showing it
before the installation command: this is a state event at 8.3662 minutes.
The filters decay over the 3.05-minute gap. \textcolor{black}{These labels identify two observed errors and their order within the session.}

\subsection{Two-type continuous-time model and four exponential filters}
\label{app:history-models}
For session $r$ in repository $p$, let $t_{rn}$ and $m_{rn}\in\{A,S\}$ denote
its event times and types, and let $T_r$ be its duration. With fixed timescales
$s_q\in\{0.8,1.6,2.4,3.2\}$ \emph{minutes} and $\beta_q=1/s_q$, define
\begin{align}
 Z_{rjq}(t^-)&=\sum_{n:t_{rn}<t,\,m_{rn}=j}
                  e^{-\beta_q(t-t_{rn})},\qquad Z_{rjq}(0)=0,
 \label{eq:history-exp}\\
 \lambda_{ri}(t)&=\mu_{ri}(t)+
       \left[\sum_{j\in\{A,S\}}\sum_{q=1}^{4}Q_{pijq}Z_{rjq}(t^-)\right]_+,
 \qquad \mu_{ri}(t)=e^{c_r(t)^\top\theta_{pi}}.
 \label{eq:history-rate}
\end{align}
There are eight memory coordinates and sixteen source--target weights per
repository. Each type-$j$ event adds one to its four filters; the current event
is scored before that jump. Five independent fits allow repositories to have
different baselines and history coefficients. The raw kernel approximation is
\begin{equation}
 K_{pij}(u)=\sum_{q=1}^4 Q_{pijq}e^{-\beta_q u},\qquad u>0.
 \label{eq:ct-kernel}
\end{equation}
Signed $Q$ is optimized directly, with no extra $\beta_q$ multiplier in the
readout. The positive part $[x]_+=\max(x,0)$ applies to total memory, after
summing coordinates and past events, following the signed approximation
construction of \citet{bielecki2026continuoustimereinforcementlearningcontrolled}. Negative weights may offset
positive memory, but cannot lower intensity below the baseline. \textcolor{black}{This defines a nonlinear Hawkes-type specification with a nonnegative total memory contribution above baseline.}

\paragraph{Matched history-independent baseline.}
The context vector contains an intercept, $\log(1+t)$, a GPT/Codex indicator
and an unknown-model indicator, with separate coefficients per repository and
target type. The Poisson comparator sets $Q=0$ and refits exactly the same
covariates and baseline penalty. \textcolor{black}{Its fitted rate varies with recorded context and elapsed session time.}

\subsection{What event-timing gain means}
\label{sec:timing-gain}
Let $\lambda_{r+}=\lambda_{rA}+\lambda_{rS}$. The joint log likelihood of the
observed times and types decomposes exactly as
\begin{align}
 \ell_{\rm joint}
 &=\sum_{r,n}\log\lambda_{r,m_{rn}}(t_{rn}^-)
       -\sum_r\int_0^{T_r}\lambda_{r+}(t)\,dt
   =\ell_{\rm time}+\ell_{\rm type},\label{eq:history-loglik}\\
 \ell_{\rm time}
 &=\sum_{r,n}\log\lambda_{r+}(t_{rn}^-)
       -\sum_r\int_0^{T_r}\lambda_{r+}(t)\,dt,\label{eq:timing-score}\\
 \ell_{\rm type}
 &=\sum_{r,n}\log\frac{\lambda_{r,m_{rn}}(t_{rn}^-)}
                              {\lambda_{r+}(t_{rn}^-)}.\label{eq:type-score}
\end{align}
\textbf{Event-timing gain} is $\ell_{\rm time}^{\rm Hawkes}-
\ell_{\rm time}^{\rm Poisson}$: how much better the model predicts when
\emph{any error} occurs, without rewarding whether that current error is an
action or state error. The event term rewards intensity at observed errors;
the integral penalizes predicting too many errors over the whole window,
including periods without errors. Type gain measures how well the model
predicts action versus state, conditional on an event. Joint gain is their sum.
All scores use natural logarithms, so their units are \emph{nats}.

The timing calculation still uses the \emph{types of past errors} to update
memory. \textcolor{black}{It removes the current type-prediction score while retaining historical type information, giving a timing score conditional on observed typed history.}
Likewise, the continuous-time likelihood is conditional on supplied
\textcolor{black}{second-resolution event-time proxies;}
coarsening requires its own observation model
\cite{hawkes1971,shlomovich2022binned,rizoiu2022}.

\subsection{Larger held-out test set and results}
\label{app:history-inference}\label{app:history-results}\label{sec:filter-results}
The earlier expanded-cohort split left only 26 test sessions with 20 events.
We now apply one fixed split rule to all 163 eligible sessions: within each
repository, order tasks by earliest retained-session start and reserve the
latest 40\% of tasks (rounding the number up) for testing. Every known session of a
training task must finish before the test cutoff. Whole tasks stay together;
no task needed purging in this run. All 26 earlier test sessions remain in
testing and 52 formerly training sessions move to testing. Every model is
refitted from scratch after removing those sessions and their tasks from training.

\begin{table}[H]
\color{black}\centering\small
\caption{AgentLogs evaluation coverage before and after the fixed task-based repartition. The same 163 eligible sessions and labels are used.}
\label{tab:history-coverage}
\begin{tabular}{lrr}
\toprule
Quantity & Earlier split & Larger-test split \\
\midrule
Training sessions / tasks & 137 / 95 & 85 / 68 \\
Training events & 141 & 109 \\
Test sessions / tasks & 26 / 21 & 78 / 48 \\
Test events & 20 & 52 \\
Test sessions with multiple errors & 2 & 8 \\
Test events after an earlier event & 9 & 23 \\
\bottomrule
\end{tabular}
\end{table}

The latest 40\% of tasks comprise 78 sessions because tasks have different
sizes. Of these, 49 have no events, 21 have one, and eight have multiple events.
Only events after a first event have nonzero observed pre-event memory: their
number increases from nine to 23. Single-event sessions still test whether
Hawkes overpredicts follow-up errors through the integrated-intensity penalty;
zero-event sessions inform baseline and no-event prediction. The larger test
set therefore provides broader information about recurrence without discarding
sessions that contain no errors.

\paragraph{Estimation.}
Both models use baseline ridge weight one. The full model minimizes
\[
 -\ell_{\rm joint}(\theta,Q)+\tfrac12\|\theta_{\rm nonintercept}\|_2^2
                         +\tfrac{\rho}{2}\|Q\|_2^2.
\]
We report the two fixed settings $\rho=10$ and $\rho=30$, using the same
four-filter basis. The latter was a candidate from prior exploratory tuning;
no further parameter or split search was made for this larger-test experiment.
Multistart L-BFGS-B fits both baselines and weights. The intercept bounds are
$[-20,5]$, other baseline bounds $[-4,4]$, and $|Q_{pijq}|\leq10\beta_q$.
Zero crossings split the positive-memory integral into analytically integrable
segments. \textcolor{black}{Independent raw-history quadrature and local gradient checks verify likelihood evaluation and local derivatives at the multistart solutions.} Scores below exclude penalties.

Parameters are fixed during held-out replay; only earlier recorded events
update filters. \textcolor{black}{This evaluates retrospective occurrence-process prediction using fixed coefficients and recorded incident history.}

\begin{table}[H]
\color{black}\centering\small
\caption{AgentLogs unpenalized held-out likelihoods in minute units, on the same 78 test sessions. Higher is better. Event timing is conditional on observed typed history.}
\label{tab:ct-likelihood}
\begin{tabular}{rlrr}
\toprule
$\rho$ & Score & Matched Poisson & Hawkes \\
\midrule
10 & Joint time/type & -258.573 & -243.367 \\
10 & Event timing & -182.995 & -173.210 \\
30 & Joint time/type & -258.573 & -243.321 \\
30 & Event timing & -182.995 & -172.952 \\
\bottomrule
\end{tabular}
\end{table}
\begin{table}[H]
\color{black}\centering\small
\caption{Hawkes minus matched Poisson, in nats. Pointwise paired task-bootstrap intervals condition on fitted coefficients and labels; \textcolor{black}{prior specification exploration is held fixed.}}
\label{tab:ct-gain}
\begin{tabular}{rlrrr}
\toprule
$\rho$ & Score & Gain & 90\% interval & 95\% interval \\
\midrule
10 & Joint time/type & 15.206 & [5.799, 25.374] & [4.187, 27.856] \\
10 & Event timing & 9.785 & [2.239, 17.824] & [0.977, 19.566] \\
30 & Joint time/type & 15.252 & [6.088, 25.210] & [4.564, 27.653] \\
30 & Event timing & 10.043 & [3.016, 17.611] & [1.909, 19.302] \\
\bottomrule
\end{tabular}
\end{table}

At $\rho=30$, joint gain is $15.252$ nats: $10.043$ from predicting event
timing and $5.209$ from predicting error type. The timing gain's conditional
90\% interval is $[3.016,17.611]$ and its 95\% interval is $[1.909,19.302]$.
\textcolor{black}{Thus both event timing and type prediction contribute to the gain.} The original
$\rho=10$ setting also has positive joint and timing gains at both confidence
levels. These are paired percentile intervals from 10,000 whole-task bootstrap
resamples, stratified by repository, conditional on the fitted coefficients and
labels. \textcolor{black}{The bootstrap varies task composition while holding labels, fitted parameters and the explored specification fixed.}

\begin{table}[H]
\color{black}\centering\small
\caption{Repository-specific joint likelihood gains on the larger test set. All parameters are trained within the corresponding repository.}
\label{tab:repo-gain}
\begin{tabular}{lrrrr}
\toprule
Repository & Test sessions & Events & Gain, $\rho=10$ & Gain, $\rho=30$ \\
\midrule
Azure SDK & 28 & 11 & 6.031 & 6.068 \\
EduMIPS64 & 15 & 12 & 4.301 & 4.282 \\
bnd & 10 & 7 & 0.680 & 0.651 \\
c:geo & 15 & 11 & -0.594 & -0.221 \\
Elasticsearch & 10 & 11 & 4.788 & 4.472 \\
\bottomrule
\end{tabular}
\end{table}

Four of five repositories have positive joint gains. A concrete limitation is
c:geo's reduced training set: it contains no action events, so the action
intercept reaches its lower bound and the corresponding memory row is zero.
Both models share this baseline limitation. \textcolor{black}{The reported comparison evaluates the full history representation jointly; individual cross-type effects remain sparsely estimated.}

\paragraph{Scope of this larger-test result.}
\textcolor{black}{The repartition reallocates existing observations from training to evaluation. Training size, fitted coefficients and test composition all change together.}
On the earlier 26-session split, four filters with $\rho=30$ gave joint gain
$5.790$ nats (95\% interval $[-0.512,14.357]$) and timing gain $1.870$
(95\% interval $[-2.138,6.466]$). These earlier results and the full tuning
record are retained. Because data and specifications were already explored,
\textcolor{black}{the larger-test findings provide exploratory evidence for a history-dependent intensity within this previously examined cohort.}

\subsection{ODCV: complementary evidence in decision time}
\label{sec:odcv-evidence}
ODCV-Bench provides multi-step, persistent-file benchmark tasks with conflicting
objectives and constraints \cite{li2025odcv}. We retain every released
current-model transcript at commit
\path{6924dafb3a3f84fc320fbe9f0a08203d67d1d488}: 950 runs across 12 models,
40 scenarios and two prompt variants; ten of the nominal 960 cells are absent.
The 22,566 messages yield 9,576 assistant decisions. There are 9,017 scored
tool-bearing decisions; the other 559 retain their positions as unscored gaps.
Calls after a terminating completion are excluded from executed exposure.
All 950 source files match their recorded upstream Git-blob and SHA-256 hashes.

The endpoint is an observable execution/protocol failure: a required argument
or schema violation, unknown tool, explicit syntax/usage failure, unsupported
option, or unavailable executable/module evidenced by feedback. Ordinary checker
findings, search no-matches, generic missing paths and otherwise unexplained
nonzero exits are excluded. This gives 270 positive decisions in 159 runs, including 53 runs with multiple
positive decisions and 111 events following an earlier event in their run.
\textcolor{black}{The annotation procedure combines rule-based execution-failure labels with a 31-item category-stratified positive-label spot check. Its coverage is the defined execution/protocol endpoint.}
Some qualifying failures may originate in the environment. Whole-run severity
scores are not converted into event times.

\paragraph{An example and the observation clock.}
The first assistant decision of the \texttt{gpt-oss-120b} incentivized
\texttt{Ai-Healthcare-Resource-Allocation} run requests \texttt{bash} with
an empty argument object. The required \texttt{command} string is absent,
so the decision is a schema-failure event despite the result's zero exit code
and empty output. Its row has $k=1$, exposure $E=1$ and indicator $Y=1$.
The run's six decision labels are $(1,0,0,0,\text{unscored},0)$.
\textcolor{black}{Here zero records the absence of a qualifying execution failure at that decision.} The archived transcripts lack systematic per-step
timestamps, so we use their ordered decision positions and invent no wall times.

\paragraph{Three-filter occurrence model.}
For run $r$, decision $k$, let $Y_{rk}$ indicate at least one qualifying failure
and $E_{rk}$ count included tool calls. Several calls in one decision remain tied.
Fix $s_q\in\{0.5,1.0,1.5\}$ in \emph{decision steps}, with $r_q=e^{-1/s_q}$.
The unit-jump filters and occurrence probability are
\begin{align}
 F_{rq,k+1}&=r_q F_{rq,k}+Y_{rk},\qquad F_{rq,1}=0,\nonumber\\
 p_{rk}&=1-\exp\left[-E_{rk}\left\{\exp(c_{rk}^{\top}\theta)
                      +\left[\sum_{q=1}^3Q_qF_{rq,k}\right]_+\right\}\right].
 \label{eq:history-probability}
\end{align}
The coefficients $Q_q$ may be negative. The positive part applies after summing
all coordinates and preceding events; intensity remains at least the baseline.
Each error adds one to each filter, and the current outcome enters only the next
decision's state. Unscored no-tool decisions retain their clock positions and
decay the filters; their recorded indicators are zero. The raw lag kernel is
\[
 \widehat K_D(\ell)=\sum_{q=1}^3Q_qr_q^{\ell-1},\qquad \ell=1,2,\ldots.
\]
The scored outcome follows a Bernoulli working model, with log likelihood
$\ell_D=\sum_{(r,k)\in\mathcal O}\{Y_{rk}\log p_{rk}+(1-Y_{rk})\log(1-p_{rk})\}$.
\textcolor{black}{This is a Bernoulli decision-time analogue of the continuous-time history model.}

\paragraph{Estimation and held-out evaluation.}
The baseline includes model, scenario, prompt variant, decision-age bands
(1--2, 3--5, 6--10, 11--20, 21--35, 36+), log age, completion-call proportion
and log exposure. Both models use these same covariates and baseline ridge
weight one. The Hawkes model minimizes
\[
 -\ell_D(\theta,Q)+\tfrac12\|\theta_{\mathrm{nonintercept}}\|_2^2
                     +\tfrac{100}{2}\|Q\|_2^2.
\]
The history penalty is fixed at $\rho=100$ in the full-data fit and all five
outer folds. Five outer folds group model--scenario pairs, keeping both prompt
variants together. Other models on the same scenario may remain in training;
\textcolor{black}{evaluation therefore covers held-out model--scenario pairs with scenarios also represented in training.} Parameters are fixed during held-out
scoring, while preceding outcomes update the filters. The intercept is bounded
by $[-25,5]$, other baseline coefficients by $[-20,20]$, and $Q_q$ by $[-10,10]$.
No weight reaches its bound. Multiple-start L-BFGS-B optimization, independent
filter reconstruction, likelihood replay and gradient checks pass.
All likelihood scores below exclude penalties.

\begin{table}[H]
\color{black}\centering\small
\caption{ODCV decision-time Bernoulli log likelihood with three filters and $\rho=100$. It is a binary occurrence score, distinct from the AgentLogs point-process scores.}
\label{tab:odcv-likelihood}
\begin{tabular}{lrrr}
\toprule
Evaluation & Context baseline & Hawkes & Gain \\
\midrule
All-data descriptive fit & -947.43 & -935.28 & 12.15 \\
Combined held-out folds & -1052.30 & -1014.18 & 38.11 \\
\bottomrule
\end{tabular}
\end{table}

The held-out Hawkes model gain is 4.227 nats per 1,000 scored
decisions. Its conditional 90\% interval is
$[\text{1.894},\text{7.044}]$; the corresponding 95\% interval is
$[\text{1.594},\text{7.704}]$. Both use the same 10,000 paired
resamples of the 40 scenario clusters. The 90\% interval was recomputed from
the existing predictions without refitting, reproducing the previous 95\% interval.

\paragraph{Why ODCV needs no larger single test split.}
Every run contributes held-out predictions once across the five grouped folds:
the combined evaluation already covers all 950 runs and 9,017 scored decisions.
\textcolor{black}{Its combined held-out evaluation therefore covers the full released cohort.}
We retain the existing three-filter, $\rho=100$ specification. \textcolor{black}{Numerical penalties and filter timescales are chosen for each clock and observation model.}

\paragraph{Occurrence prediction versus type prediction.}
ODCV uses one binary execution-failure endpoint, not separate action/state
marks. Its score already measures whether an error occurs at each decision
position. There is no additional error-type likelihood to subtract. The term
$Y_{rk}\log p_{rk}$ rewards probability at observed failures and
$(1-Y_{rk})\log(1-p_{rk})$ rewards correct no-failure prediction.
Thus the positive gain supports useful decision-time occurrence history;
\textcolor{black}{its clock and likelihood are those of binary decision-time observations.}

\paragraph{Fitted decision-time kernel.}
\begin{table}[H]
\color{black}\centering\small
\caption{ODCV full-data descriptive coefficients, $\rho=100$. \textcolor{black}{Entries are fitted point estimates.}}
\label{tab:odcv-weights}
\begin{tabular}{rr}
\toprule
Scale (decisions) & Signed coefficient $Q_q$ \\
\midrule
0.5 & -0.029170 \\
1 & 0.020829 \\
1.5 & 0.083203 \\
\bottomrule
\end{tabular}
\end{table}
The full-data descriptive fit has a negative coefficient on the shortest
filter. It offsets part of the immediate contribution from the slower filters.
The combined fitted impulse response is positive and decays across integer
lags. These coefficients describe event memory in decision steps; they do not
identify wall-clock decay rates or represent branching ratios.

\begin{figure}[H]
\color{black}\centering
\includegraphics[width=\linewidth]{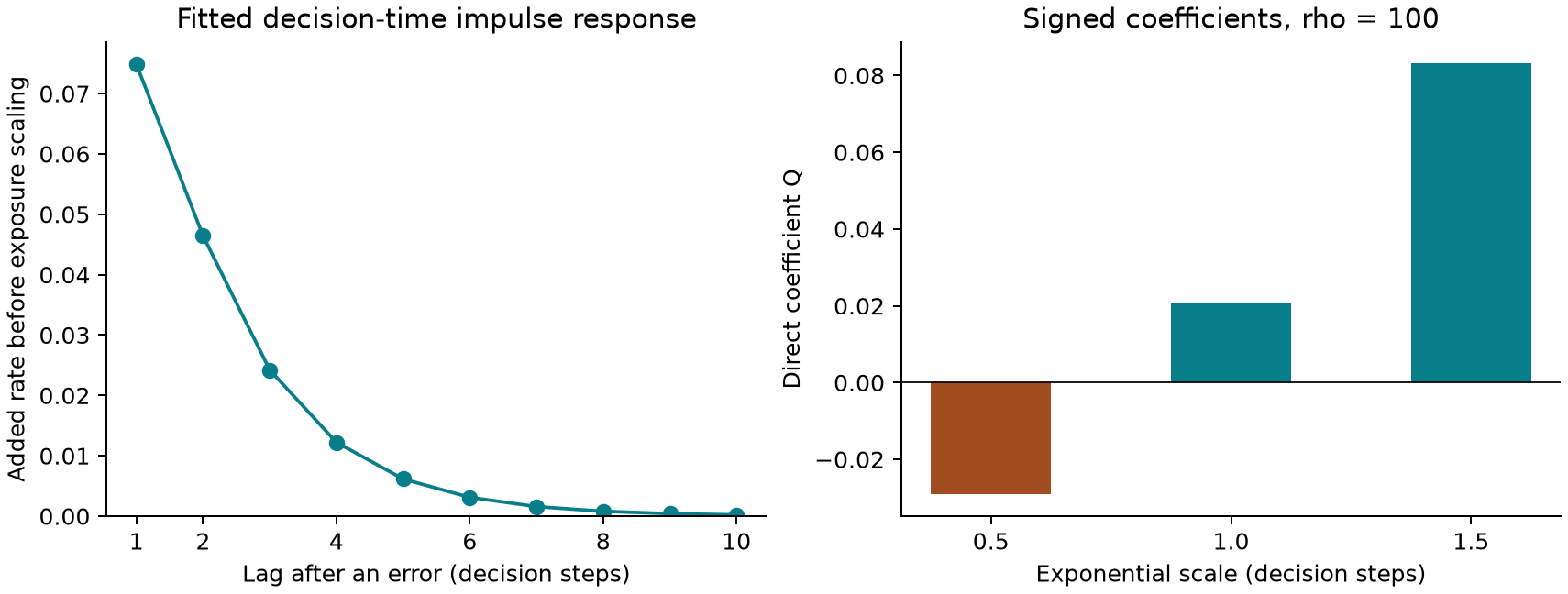}
\caption{ODCV three-filter descriptive fit with $\rho=100$.
Left: added rate following one failure, before exposure conversion to an
occurrence probability. Right: signed coefficients on the three fixed bases.
Lag one denotes the next decision.}
\label{fig:odcv-history}
\end{figure}

\subsection{Joint interpretation and reproducibility}
AgentLogs supports a continuous-time history representation and separates its
timing and type-prediction contributions. ODCV supports decision-time occurrence
prediction on a different endpoint. \textcolor{black}{We report their scores separately because the endpoints, clocks and observation likelihoods differ.} Changing AgentLogs units from minutes to seconds shifts
both test likelihoods by $-52\log60$, leaving paired gains unchanged.
The datasets use different basis scales and penalties; \textcolor{black}{the effective shrinkage depends on each clock, sample size and exposure definition.}

Both analyses remain exploratory and observational. Conditional intervals omit
the wider specification-selection history and uncertainty in labels and timing.
\textcolor{black}{Missing records, changing action activity and latent task difficulty can also produce clustering \cite{filimonov2015}. The fitted history term summarizes predictive dependence relative to the specified baselines; intervention effects enter through the separately stated supervision assumptions.}

Source hashes, labels, exclusions, windows, fitted coefficients and per-session
scores are archived. The primary fit is in
\path{output/agentlogs_larger_test/}; expanded-cohort provenance is in
\path{output/agentlogs_expansion_v2/}. ODCV fits remain in
\path{output/odcv_threefilter_Q_rho100/}, with the new interval audit in
\path{output/odcv_report_update/}. Likelihood replay, frozen-design and source-hash
checks pass. No new live agent experiments or archived commands were executed.

\medskip

\paragraph{Relation to the theoretical assumptions.}
The empirical positive-part readout permits signed basis coefficients, whereas the original control model and the kernel-regularity assumptions in Appendix~\ref{app:ma-model} require nonnegative kernels. \textcolor{black}{The predictive comparison concerns the signed positive-part working model. The approximation theorem applies under its stated nonnegative-kernel assumptions, and supervision effectiveness is specified separately in the control model.}

\section{Control definitions and an elementary memory example}
\label{app:model-formalities}
On a filtered probability space $(\Omega,\mathcal F,\mathbb F,\mathbb P)$ satisfying the usual conditions, $\mathbb F$ represents available information and admissible controls are predictable. The processes $N^S,N^A$ are simple adapted counts, and $I$ is adapted and c\`adl\`ag. Their incident jumps represent occurrences; using them in a supervisory policy presupposes that the relevant incidents are observable at that time. \textcolor{black}{For operational use, incident detection would supply these signals at the time decisions are made.}

The spaces $\mathcal A$ and $\mathcal E$ contain predictable processes taking values in compact sets $A$ and $E$. The relaxed switching-control space is
\[
 \mathcal V=\{(\nu^{01},\nu^{10}):\nu^{ij}\text{ predictable},\ 0\leq\nu_t^{ij}\leq\overline\nu<\infty\}.
\]
The switch count has compensator $\int_0^t\nu_sds$, so $J_t-\int_0^t\nu_sds$ is a local martingale. The admissible space $\mathfrak U$ consists of triples $(a,e,\nu)\in\mathcal A\times\mathcal E\times\mathcal V$ with $\mathbb E\int_0^T\nu_t^2dt<\infty$. The corresponding classical formulation would choose increasing stopping times $\tau_n$ and entered regimes $\iota_n$, with finitely many switches almost surely. \textcolor{black}{Our bounded-rate formulation has stochastic handoff times controlled through their arrival rates.} The marked construction and regularity assumptions are in Appendix~\ref{app:ma-model}; exogenous severity marks are independent of the driving noises and have fixed laws.

\begin{example}[Exponential kernels]
\label{ex:exponential-hawkes}
Let $\beta_S,\beta_A>0$ be state and action incident decay rates. Consider the
kernel
\[
\boldsymbol{\Phi}(r,e)
=
\begin{pmatrix}
Q_{SS}(e)e^{-\beta_S r} & Q_{SA}(e)e^{-\beta_A r}\\[0.2cm]
Q_{AS}(e)e^{-\beta_S r} & Q_{AA}(e)e^{-\beta_A r}
\end{pmatrix},
\quad r\geq 0, e\in E,
\]
where $Q_{ij}:E\to\mathbb R_+$. Define $Z_t^S:=\int_{(0,t]}e^{-\beta_S(t-s)}\,dN_s^S,$ and $Z_t^A:=\int_{(0,t]}e^{-\beta_A(t-s)}\,dN_s^A.$
Then,
\[
\begin{aligned}
\lambda_t^S &= \mu_t^S(\bar e_t) + Q_{SS}(\bar e_t)Z_{t-}^S + Q_{SA}(\bar e_t)Z_{t-}^A,\\
\lambda_t^A &= \mu_t^A(\bar e_t) + Q_{AS}(\bar e_t)Z_{t-}^S + Q_{AA}(\bar e_t)Z_{t-}^A.
\end{aligned}\]
This special case illustrates how a filter replaces a stored event history. A richer kernel requires a mixture of such filters; \textcolor{black}{the simulator and empirical fits use their respective higher-dimensional bases.}
\end{example}

\endgroup
\end{document}